%% file: main.tex
\documentclass[11pt,oneside]{article}
\usepackage{graphicx}

\usepackage{amsfonts,bm,enumerate}
\usepackage{amsthm,amsmath,amssymb,amsfonts}
\allowdisplaybreaks
\usepackage{mathtools,xfrac}
\usepackage{url,xcolor}
\usepackage{graphicx,epsfig}
\usepackage{enumerate}
\usepackage{pdfsync,array}
\usepackage[normalem]{ulem}

\usepackage[sort,numbers]{natbib}

\usepackage[ruled,linesnumbered]{algorithm2e}
\usepackage{verbatim}
\usepackage{makeidx,latexsym}
\usepackage[colorlinks=true,citecolor=blue]{hyperref}
\usepackage{booktabs}

\usepackage{slashbox}

\usepackage{tikz}
\usetikzlibrary{calc}
\usetikzlibrary{patterns}
\usepackage{pgfplots,pgfplotstable}

\usepackage{caption,subcaption}

\usepackage[capitalize]{cleveref}

\newcommand{\eps}{\epsilon}

\theoremstyle{plain}
\newtheorem{theorem}{Theorem}[section]

\newtheorem{lemma}[theorem]{Lemma}
\newtheorem*{lemma*}{Lemma}
\newtheorem{proposition}[theorem]{Proposition}
\newtheorem*{proposition*}{Proposition}

\newtheorem*{observation*}{Observation}

\theoremstyle{definition}

\newtheorem{definition}{Definition}[section]
\newtheorem{assumption}{Assumption}[section]

\theoremstyle{remark}
\newtheorem{remark}{Remark}[section]
\newtheorem*{remark*}{Remark}

\newtheorem*{example*}{Example}

\input{Def}

\usepackage{authblk}
\usepackage[margin=1in]{geometry}
\usepackage{setspace}

\title{Adversarial Classification via Distributional Robustness with Wasserstein Ambiguity}
\author[1]{Nam Ho-Nguyen}
\author[2]{Stephen J. Wright}
\affil[1]{The University of Sydney}
\affil[2]{University of Wisconsin--Madison}

\begin{document}

\maketitle

\begin{abstract}
We study a model for adversarial classification based on distributionally robust chance constraints. We show that under Wasserstein ambiguity, the model aims to minimize the conditional value-at-risk of the distance to misclassification, and we explore links to adversarial classification models proposed earlier and to maximum-margin classifiers. We also provide a reformulation of the distributionally robust model for linear classification, and show it is equivalent to minimizing a regularized ramp loss objective. Numerical experiments show that, despite the nonconvexity of this formulation, standard descent methods appear to converge to the global minimizer for this problem. Inspired by this observation, we show that, for a certain class of distributions, the only stationary point of the regularized ramp loss minimization problem is the global minimizer.
\end{abstract}

\input{intro}

\input{margin}

\input{reformulation}

\input{numerical}

\input{localmin}

\bibliographystyle{abbrvnat}

\input{main.bbl}
\end{document}

%% file: Def.tex
\definecolor{gold}{rgb}{0.85,0.65,0}
\colorlet{dgreen}{green!50!black}

\def\beq{\begin{equation}}
\def\eeq{\end{equation}}

\def\fnote#1{\footnote}
\newcommand{\BlBox}{\rule{1.5ex}{1.5ex}}
\newcommand{\epr}{\hfill\BlBox\vspace{0.0cm}\par}

\newcommand{\grad}{\ensuremath{\nabla}}

\def\R{{\mathbb{R}}}

\def\cI{{\cal I}}

\def\cR{{\cal R}}

\def\cT{{\cal T}}

\def\cW{{\cal W}}

\newcommand{\bbE}{\mathbb{E}}

\newcommand{\bbI}{\mathbb{I}}

\newcommand{\bbN}{\mathbb{N}}

\newcommand{\bbP}{\mathbb{P}}

\newcommand{\bbR}{\mathbb{R}}

\newcommand{\CI}{\mathcal{I}}

\newcommand{\CP}{\mathcal{P}}

\newcommand{\CW}{\mathcal{W}}

\DeclareMathOperator*{\argmin}{arg\,min}
\DeclareMathOperator*{\argmax}{arg\,max}

\DeclareMathOperator{\sign}{sign}

\DeclareMathOperator{\CVaR}{CVaR}

\DeclareMathOperator{\Span}{Span}

\DeclareMathOperator{\cl}{cl}

\DeclareMathOperator{\Conv}{Conv}

\DeclareMathOperator{\Cone}{Cone}

\def\log{\mathop{{\rm log}}}

\newcommand{\psisig}{\psi_{\sigma}}
\newcommand{\beps}{\bar\epsilon}

%% file: intro.tex
\section{Introduction} \label{sec:intro}

Optimization models have been used for prediction and pattern
recognition in data analysis as early as the work of
\citet{Mangasarian1965} in the 1960s. Recent developments have seen
models grow in size and complexity, with success on a variety of
tasks, which has spurred many practical and theoretical advances in
data science. However, it has been observed that models that achieve
remarkable prediction accuracy on unseen data can lack robustness to
small perturbations of the data
\citep{SzegedyEtAl2014,GoodfellowEtAl2015}.
For example, the correct classification of a data point for a
trained model can often be switched to incorrect by adding a small
perturbation, carefully chosen.  This fact is particularly problematic
for image classification tasks, where the perturbation that
yields misclassification can be imperceptible to the human
eye\footnote{See, for example,
	\url{https://adversarial-ml-tutorial.org/introduction/}.}.

This observation has led to the emergence of \emph{adversarial machine learning}, a field that examines robustness properties of models to (potentially adversarial) data perturbations. 
Two streams of work in this area are particularly notable. 
The first is \emph{adversarial attack} \citep{CarliniWagner2017a,CarliniWagner2017b,MoosaviDezfooliEtAl2016}, where the aim is to ``fool'' a trained model by constructing adversarial perturbations. 
The second is \emph{adversarial defense}, which focuses on model training methods that produce classifiers that are robust to perturbations \citep{CohenEtAl2019,MoosaviDezfooliEtAl2018,WangEtAl2019,MadryEtAl2018,TramerEtAl2018,ChenEtAl2019neurips,WongKolter2018,BertsimasEtAl2019}. 
Most models for adversarial defense are based on robust optimization, where the training error is minimized subject to arbitrary perturbations of the data in a ball defined by some distance function (for example, a norm in feature space). As such, these algorithms are reminiscent of iterative algorithms from robust optimization \citep{HoNguyenKK2018,MutapcicBoyd2009,BenTalEtAl2015}. 
Theoretical works on adversarial defense also focus on the robust optimization model, discussing several important topics such as hardness and fundamental limits \citep{GilmerEtAl2019,BubeckEtAl2019,FawziEtAl2018,FawziEtAl2018neurips}, learnability and risk bounds \citep{YinEtAl2019,ZhangEtAl2019}, as well as margin guarantees and implicit bias for specific algorithms \citep{CharlesEtAl2019,LiEtAl2020Implicit}.

In optimization under uncertainty and data-driven decision-making, the concept of \emph{distributional robustness} offers an intriguing alternative to stochastic optimization and robust optimization \citep{MohajerinEsfahaniKuhn2018,BlanchetMurthy2019,ChenKuhnWiesemann2018,Xie2019}. 
Instead of considering perturbations of the data (as in robust optimization), this approach considers perturbations in the \emph{space of distributions} from which the data is drawn, according to some distance measure in distribution space (for example, $\phi$-divergence or Wasserstein distance). 
This technique enjoys strong statistical guarantees and its numerical performance often outperforms models based on stochastic or robust optimization. 
In particular, for perturbations based on Wasserstein distances, the new distributions need not have the same support as the original empirical distribution.

In this paper, we explore adversarial defense by using ideas from
distributional robustness and Wasserstein ambiguity sets. We
focus on the fundamental classification problem in machine learning,
and its formulation as an optimization problem in which we seek to
minimize the probability of misclassification. We study a
distributionally robust version of this problem and explore
connections between maximum-margin classifiers and conditional
value-at-risk objectives. We then focus on the linear classification
problem. While convex linear classification formulations are
well-known \citep{BennettMangasarian1992}, the model we study is
based on a ``zero-one'' loss function $r \mapsto \bm{1} (r \leq 0)$,
which is discontinuous and thus nonconvex. However, we show that in
the case of binary linear classification, the reformulation of the
distributionally robust model gives rise to the ``ramp loss''
function $L_R$ defined in \eqref{eq:ramploss}, and we propose
efficient first-order algorithms for minimization. While the ramp
loss is nonconvex, the nonconvexity is apparently ``benign''; the
global minimizer appears to be found for sufficiently dense
distributions. Indeed, we prove that for a certain class of distributions, the global
minimizer is the only stationary point. Numerical experiments
confirm this observation.

\subsection{Problem description}\label{sec:description}

Suppose that data $\xi$ is drawn from some distribution $P$ over a set $S$. 
In the learning task, we need to find a decision variable $w$, a classifier, from a space $\CW$. 
For each $(w,\xi) \in \CW \times S$, we evaluate the result of choosing classifier $w$ for outcome $\xi$ via a ``safety function'' $z:\CW \times S \to \bbR \cup \{+\infty\}$. 
We say that $w$ ``correctly classifies'' the point $\xi$ when $z(w,\xi) > 0$, and $w$ ``misclassifies'' $\xi$ when $z(w,\xi) \leq 0$. 
Thus, we would like to choose $w$ so as to minimize the probability of misclassification, that is, 
\begin{equation}\label{eq:prob-nominal}
	\inf_{w \in \CW} \bbP_{\xi \sim P} \left[ z(w,\xi) \leq 0 \right].
\end{equation}

This fundamental problem is a generalization of the binary
classification problem, which is obtained when $S = X \times
\{\pm 1\}$, where $\xi = (x,y)$ is a feature-label pair, $\CW$
describes the space of classifiers under consideration (e.g., linear
classifiers or a reproducing kernel Hilbert space), and $z(w,(x,y)) = y
w(x)$; so $w$ correctly classifies $(x,y)$ if and only if $\sign(w(x))
= y$.

In the context of adversarial classification, we are interested in finding decisions $w \in \CW$ which are \emph{robust to (potentially adversarial) perturbations of the data $\xi$}. In other words, if our chosen $w$ correctly classifies $\xi$ (that is, $z(w,\xi) > 0$), then any small perturbation $\xi+\Delta$ should also be correctly classified, that is, $z(w,\xi+\Delta)>0$ for ``sufficiently small'' $\Delta$. 
To measure the size of perturbations, we use a distance function $c:S \times S \to [0,+\infty]$ that is nonnegative and lower semicontinuous with $c(\xi,\xi') = 0$ if and only if $\xi = \xi'$. (For binary classification $\xi = (x,y)$ mentioned above, the distance function can be $c((x,y),(x'y')) = \|x-x'\|+\bbI_{y=y'}(y,y')$ where $\bbI_A$ is the convex indicator of the set $A$.) 
For a classifier $w \in \cW$, we define the
\emph{margin}, or \emph{distance to misclassification}, of a point $\xi \in S$ as
\begin{equation}\label{eq:distance}
	d(w,\xi) := \inf_{\xi' \in S} \, \left\{ c(\xi,\xi') : z(w,\xi') \leq 0 \right\}.
\end{equation}
(Note that $d(w,\xi)=0 \Leftrightarrow z(x,\xi) \le 0$.)

Two optimization models commonly studied in previous works on adversarial classification are the following:
\begin{subequations} \label{eq:prob-previous-adversarial-forms}
	\begin{align}
		&\inf_{w \in \CW} \, \bbP_{\xi \sim P} \left[ d(w,\xi) \leq \epsilon \right], \label{eq:prob-robust}\\
		&\sup_{w \in \CW} \, \bbE_{\xi \sim P} \left[ d(w,\xi) \right]. \label{eq:prob-max-dist}
	\end{align}
\end{subequations}
The first model \eqref{eq:prob-robust}, which is the most popular such model, aims to minimize the probability that the distance of a data point $\xi$ to a bad result will be smaller than a certain threshold $\epsilon \geq 0$.
This is more commonly stated as a robust optimization type problem:
\[ \inf_{w \in \cW} \bbP_{\xi \sim P} \left[ \inf_{\xi' : c(\xi,\xi') \leq \eps} z(w,\xi') \leq 0 \right]. \]
Note that \eqref{eq:prob-nominal} is a special case of \eqref{eq:prob-robust} in which $\epsilon = 0$. 
The second model \eqref{eq:prob-max-dist} maximizes the expected margin.
This model removes the need to choose a parameter $\epsilon$, but \citet[Lemma~1]{FawziEtAl2018} has shown that this measure is inversely related to the probability of misclassification $\bbP_{\xi \sim P} [z(w,\xi) \leq 0]$, that is, a lower probability of misclassification (good) leads to a lower expected distance (bad), and vice versa. 
Thus, this model is not used often.

With \emph{distributional robustness}, rather than guarding against perturbations in the data points $\xi$, we aim to guard against perturbations of the distribution $P$ of the data. 
In this paper, we study the following distributionally robust optimization (DRO) formulation (stated in two equivalent forms that will be used interchangeably throughout):
\begin{equation}\label{eq:prob-DRO}
	\inf_{w \in \CW} \, \sup_{Q:d_W(P,Q) \leq \epsilon} \, \bbP_{\xi \sim Q} \left[ z(w,\xi) \leq 0 \right] \;\; \Leftrightarrow \;\; \inf_{w \in \CW} \, \sup_{Q:d_W(P,Q) \leq \epsilon} \, \bbP_{\xi \sim Q} \left[ d(w,\xi) = 0 \right],
\end{equation}
that is, we aim to minimize the \emph{worst-case} misclassification probability over a ball of distributions $\{Q : d_W(Q,P) \leq \epsilon\}$. 
The ball is defined via the \emph{Wasserstein distance} between two distributions, which is defined via the function $c$ as follows:
\begin{equation} \label{eq:def.Wasserstein}
	d_W(P,Q) := \inf_{\Pi} \, \left\{ \bbE_{(\xi,\xi') \sim \Pi} \left[ c(\xi,\xi') \right] : \Pi \text{ has marginals } P,Q \right\}.
\end{equation}
We use the Wasserstein distance due to the fact that distributions in the Wasserstein ball can also capture perturbations in data points themselves, similar to \eqref{eq:prob-robust}. Indeed, \citet[Corollary 1]{PydiJog2021} prove that \eqref{eq:prob-robust} is upper bounded by \eqref{eq:prob-DRO}, thus the two are intimately related. One of our aims in this paper, however, is to understand the optimal solutions to \eqref{eq:prob-DRO} and how they differ to those of \eqref{eq:prob-robust}.

In practice, the true distribution $P$ is not known to us; we typically have only a finite sample  $\xi_i \sim P$, $i \in [n]$ of training data, drawn from $P$, from which we can define the empirical distribution $P_n := \frac{1}{n} \sum_{i \in [n]} \delta_{\xi_i}$. 
We use $P_n$ as the center of the ball of distributions, that is, we solve the formulation \eqref{eq:prob-DRO} in which $P_n$ replaces $P$.

\subsection{Contributions and outline} \label{sec:contributions}
In this paper, we first explore how using a Wasserstein ambiguity set in \eqref{eq:prob-DRO} can result in stronger guarantees for robustness to perturbations than \eqref{eq:prob-robust}.
Specifically, in Section~\ref{sec:margin-finite}, we show that for sufficiently small $\epsilon$, \eqref{eq:prob-DRO}
yields
the maximum-margin classifier. 
In Section~\ref{sec:margin-cvar}, we extend the link between the conditional value-at-risk of the distance function \eqref{eq:distance} and a chance-constrained version of \eqref{eq:prob-DRO} (observed by \citet{Xie2019}) to the probability minimization problem \eqref{eq:prob-DRO}. 
This link yields an
interpretation of optimal solutions of \eqref{eq:prob-DRO} for large $\epsilon$, as optimizers of the conditional value-at-risk of the distance function \eqref{eq:distance}. Thus, solving the DRO problem ensures that, for a certain fraction $\rho$, the average of the $\rho$-proportion of \emph{smallest} margins is as large as possible.

In Section~\ref{sec:reformulation}, we give a reformulation of \eqref{eq:prob-DRO} for linear classifiers, obtaining a regularized risk minimization problem with a ``ramp loss'' objective. 
This formulation highlights the link between distributional robustness and robustness to outliers, a criterion which has motivated the use of ramp loss in the past. 
We suggest a class of smooth approximations for the ramp loss, allowing problems with this objective to be solved (approximately) with standard continuous optimization algorithms.

In Section~\ref{sec:numerical}, we perform some numerical tests of
linear classification on three different data distributions.  We
observe that the regularized smoothed ramp loss minimization problem
arising from \eqref{eq:prob-DRO}, while nonconvex, is ``benign'' in
the sense that the global minimum appears to be identified easily by
smooth nonconvex optimization methods, for modest values of the
training set size $n$. We also experiment with nonseparable data sets
containing mislabelled points, showing that the problem arising from
\eqref{eq:prob-DRO} is more robust to these ``attacks'' than the
hinge-loss function often used to find the classifying hyperplane.

Motivated by the observations in Section~\ref{sec:numerical}, we prove
in Section~\ref{sec:localmin} that the ramp-loss problem indeed has
only a single stationary point (which is therefore the global
minimizer) for the class of spherically symmetric distributions.

\subsection{Related work}

There are a number of works that explore distributional robustness for
machine learning model training.  These papers consider a
distributionally robust version of empirical risk minimization, which
seeks to minimize the worst-case risk over some ambiguity set of
distributions around the empirical distribution.
\citet{LeeRaginsky2018} consider a distributionally robust ERM
problem, exploring such theoretical questions as learnability of the
minimax risk and its relationship to well-known function-class
complexity measures. Their work targets smooth loss functions, thus
does not apply to \eqref{eq:prob-DRO}.  Works that consider a
Wasserstein ambiguity, similar to \eqref{eq:prob-DRO}, include
\citet{SinhaEtAl2018,Shafieezadeh-AbadehEtAl2015,Shafieezadeh-AbadehEtAl2019,ChenPaschalidis2018};
whereas \citet{HuEtAl2018} uses a distance measure based on
$\phi$-divergences.
For Wasserstein ambiguity, \citet{SinhaEtAl2018} provide an approximation scheme for distributionally robust ERM by using the duality result of Lemma~\ref{lemma:dual-representation},
showing convergence of this scheme when the loss is smooth and the
distance $c$ used to define the Wasserstein distance in
\eqref{eq:def.Wasserstein} is strongly convex.  When the loss function
is of a ``nice'' form (e.g., logistic or hinge loss for
classification, $\ell_1$-loss for regression),
\citet{Shafieezadeh-AbadehEtAl2015,ChenPaschalidis2018,Shafieezadeh-AbadehEtAl2019,KuhnEtAl2019tutorial}
show that the incorporation of Wasserstein distributional robustness
yields a \emph{regularized} empirical risk minimization problem.  This
observation is quite similar to our results in
Section~\ref{sec:reformulation}, with a few key differences outlined
in Remark \ref{rem:reformulation-comparison}.  Also, discontinuous
losses, including the ``$0$-$1$'' loss explored in our paper, are not
considered by
\citet{SinhaEtAl2018,Shafieezadeh-AbadehEtAl2015,Shafieezadeh-AbadehEtAl2019,ChenPaschalidis2018}.
Furthermore, none of these works provide an interpretation the optimal
classifier like the one we provide in Section~\ref{sec:margin}.

In this sense, the goals of Section~\ref{sec:margin} are similar to those of \citet{HuEtAl2018}, who work with $\phi$-divergence ambiguity sets. 
Their paper shows that the formulation that incorporates $\phi$-divergence ambiguity does not result in classifiers different from those obtained by simply minimizing the empirical distribution. 
They suggest a modification of the ambiguity set and show experimental improvements over the basic $\phi$-divergence ambiguity set. 
The main difference between our work and theirs is that we consider a different (Wasserstein-based) ambiguity set, which results in an entirely different analysis and computations. Furthermore, using $\phi$-divergence ambiguity does not seem to have a strong theoretical connection with the traditional adversarial training model \eqref{eq:prob-robust}, whereas we show that the Wasserstein ambiguity \eqref{eq:prob-DRO} has close links to \eqref{eq:prob-robust}.

We mention some relevant works from the robust optimization-based models for adversarial training. 
\citet{CharlesEtAl2019} and \citet{LiEtAl2020Implicit} both provide margin guarantees for gradient descent on an adversarial logistic regression model. 
We also give margin guarantees for the distributionally robust model \eqref{eq:prob-DRO} in Section~\ref{sec:margin}, but ours are algorithm-independent, providing insight into use of the Wasserstein ambiguity set for adversarial defense. 
\citet{BertsimasCopenhaver2018} and 
\citet{XuEtAl2009,XuEtAl2011} have observed that for ``nice'' loss functions, (non-distributionally) robust models for ERM also reformulate to a regularized ERM problem.

Finally, we mention that our results concerning uniqueness of the stationary point in Section~\ref{sec:localmin} are inspired by, and are similar in spirit to, local minima results for low-rank matrix factorization (see, for example, \citet{ChiEtAl2019lowrankmatrixsurvey}).

%% file: margin.tex
\section{Margin Guarantees and Conditional Value-at-Risk}\label{sec:margin}

In this section, we highlight the relationship between the main problem \eqref{eq:prob-DRO} and a generalization of maximum-margin classifiers, as well as the conditional value-at-risk of the margin function $d(w,\xi)$.

\subsection{Margin guarantees for finite support distributions}\label{sec:margin-finite}

We start by exploring the relationship between solutions to \eqref{eq:prob-DRO} and maximum margin classifiers. 
We recall the definition \eqref{eq:distance} of \emph{margin}
$d(w,\xi)$ for any $w \in \CW$ and data point $\xi \in S$.  We say
that a classifier $w$ has a {\em margin of at least $\gamma$} if
$d(w,\xi) \geq \gamma$ for all $\xi \in S$.  When $\gamma > 0$, this
implies that a perturbation of size at most $\gamma$ (as measured by
the distance function $c$ in \eqref{eq:distance}) for any data point
$\xi$ will still be correctly classified by $w$.  In the context of
guarding against adversarial perturbations of the data, it is
clearly of interest to find a classifier $w$ with maximum
margin, that is, the one that has the largest possible
$\gamma$. On the other hand, some datasets $S$ cannot be
perfectly separated, that is, for any classifier $w \in \cW$, there
will exist some $\xi \in S$ such that $d(w,\xi) = 0$. To enable
discussion of maximum margins in both separable and non-separable
settings, we propose a generalized margin concept in
Definition~\ref{def:max-margin-bilevel} as the value of a bilevel optimization problem.  We then show that solving the DRO formulation
\eqref{eq:prob-DRO} is exactly equivalent to finding a generalized
maximum margin classifier for small enough ambiguity radius
$\epsilon$.  This highlights the fact that the Wasserstein
ambiguity set is quite natural for modeling adversarial
classification.  We work with the following assumption on $P$.
\begin{assumption}\label{ass:finite-support}
	The distribution $P$ has finite support, that is,  $P = \sum_{i \in [n]} p_i \delta_{\xi_i}$, where each $p_i > 0$ and $\sum_{i \in [n]} p_i = 1$.
\end{assumption}
We make this assumption because for most {\em continuous}
distributions, even our generalization of the margin will always be
$0$, so that a discussion of margin for such distributions is not
meaningful. Since any training or test set we encounter in
practice is finite, the finite-support case is worth our focus.

Under Assumption~\ref{ass:finite-support}, we define the notion of \emph{generalized margin} of $P$. 
For $w \in \CW$ and $\rho \in [0,1]$, we define
\begin{align*}
	I(w) &:= \{i \in [n] : d(w,\xi_i) = 0\} \;\\
	&\quad \mbox{\rm (points misclassified by $w$)} \\
	\CI(\rho) &:= \left\{ I \subseteq [n] : \sum_{i \in I} p_i \leq \rho \right\}
	\;\\
	&\quad \mbox{\rm (subsets of $[n]$ with cumulative probability at most $\rho$)} \\
	\eta(w) &:= \min_{i \in [n] \setminus I(w)} d(w,\xi_i) \;\\
	&\quad \mbox{\rm (margin of $w$ with misclassified points excluded)} \\
	\gamma(\rho) &:= \sup_{w \in \CW} \left\{ \eta(w) : I(w) \in \CI(\rho) \right\} \;\\
	&\quad \mbox{\rm (max. margin with at most fraction $\le \rho$ of points  misclassified).}
\end{align*}
The usual concept of margin is $\gamma(0)$.
Given these quantities, we define the generalized maximum margin of $P$ with respect to the classifiers $\cW$ as the value of the following bilevel optimization problem.
\begin{definition}\label{def:max-margin-bilevel}
	Given $P$ and $\cW$, the \emph{generalized maximum margin} is defined to be
	\begin{equation}\label{eq:margin-bilevel}
		\gamma^* := \sup_{w \in \cW} \left\{ \eta(w) : w \in \argmin_{w' \in \CW} \bbP_{\xi \sim P}[d(w',\xi)=0] \right\}.
	\end{equation}
\end{definition}
Note that Definition \ref{def:max-margin-bilevel} implicitly assumes
that the $\arg\min$ over $w' \in \CW$ is achieved in
\eqref{eq:margin-bilevel}.
We show that under Assumption \ref{ass:finite-support}, this is indeed
the case, and furthermore that $\gamma^* > 0$.
\begin{proposition}\label{prop:margin-bilevel-solvable}
	Suppose Assumption \ref{ass:finite-support} holds. Define
	\begin{equation}\label{eq:rho-optimal}
		\rho^* := \inf\left\{ \rho \in [0,1] : \gamma(\rho) > 0 \right\}.
	\end{equation}
	Then $\rho^* = \inf_{w' \in \CW} \bbP_{\xi \sim P}[d(w',\xi)=0]$, there exists $w \in \cW$ such that
	\[\bbP_{\xi \sim P}[d(w,\xi)=0] = \rho^*, \text{ and } \gamma^* = \gamma(\rho^*) > 0.\]
\end{proposition}
\begin{proof}
	We first prove that under Assumption \ref{ass:finite-support}, the function $\rho \mapsto \gamma(\rho)$ is a right-continuous non-decreasing step function. The fact that $\gamma(\rho)$ is non-decreasing follows since $\CI(\rho) \subseteq \CI(\rho')$ for $\rho \leq \rho'$.	
	To show that it is a right-continuous step function, consider the finite set of all possible probability sums $\CP = \left\{ \sum_{i \in I} p_i : I \subseteq [n] \right\} \subset [0,1]$. 
	Let us order $\CP$ as $\CP = \left\{ \rho^1,\ldots,\rho^K \right\}$ where $\rho^1 < \cdots < \rho^K$. 
	There is no configuration $I \subseteq [n]$ such that $\rho^k < \sum_{i \in I} p_i < \rho^{k+1}$.
	Thus, $\CI(\rho) = \CI(\rho^k)$ and hence $\gamma(\rho) = \gamma(\rho^k)$ for $\rho \in [\rho_k, \rho_{k+1})$, proving the claim.
	
	To prove the proposition, first note that there exists no classifier $w \in \CW$ such that $\bbP_{\xi \sim P}[d(w,\xi)=0] = \underline{\rho} < \rho^*$, otherwise, we have $\gamma(\underline{\rho}) \geq \eta(w) > 0$, contradicting the definition of $\rho^*$. This shows that $\rho^* \leq \inf_{w' \in \CW} \bbP_{\xi \sim P}[d(w',\xi)=0]$. By the description of $\rho^*$ as the infimal $\rho$ such that $\gamma(\rho) > 0$ and by the right-continuity of $\gamma(\cdot)$ and the fact that $\gamma(\cdot)$ is a step function, we must have $\gamma(\rho^*) > 0$.
	Since $\gamma(\rho^*) > 0$, there must exist some $w \in \cW$ such that $I(w) \in \cI(\rho^*)$, that is, $\sum_{i \in I(w)} p_i = \bbP_{\xi \sim P} [d(w,\xi)=0] \leq \rho^*$. Since we cannot have $\bbP_{\xi \sim P} [d(w,\xi)=0] < \rho^*$ we conclude that $\bbP_{\xi \sim P} [d(w,\xi)=0] = \rho^*$. Therefore $\inf_{w \in \CW} \bbP_{\xi \sim P} [d(w,\xi)=0] = \rho^*$. Furthermore, by definition, we have $\gamma^* = \gamma(\rho^*) > 0$.
\end{proof}

The following result gives a precise characterization of the worst-case misclassification probability of a classifier $w$ for a radius $\epsilon$ that is smaller than the probability-weighted margin of $w$. It also gives a lower bound on worst-case error probability when $\epsilon$ is larger than this quantity. 
\begin{proposition}\label{prop:margin-finite-bounds}
	Under Assumption~\ref{ass:finite-support}, for $w \in \CW$ such that
	\[\epsilon \leq \min_{i \in [n] \setminus I(w)} d(w,\xi_i) p_i,\]
	we have
	\[
	\sup_{Q:d_W(P,Q) \leq \epsilon} \bbP_{\xi \sim Q}[d(w,\xi)=0] = \sum_{i \in I(w)} p_i + \frac{\epsilon}{\eta(w)}.
	\]
	For $w \in \CW$ such that $\epsilon > \min_{i \in [n] \setminus I(w)} d(w,\xi_i) p_i$, we have
	\[
	\sup_{Q:d_W(P,Q) \leq \epsilon} \bbP_{\xi \sim Q}[d(w,\xi)=0] > \sum_{i \in I(w)} p_i + \min_{i \in [n] \setminus I(w)} p_i.
	\]
\end{proposition}
To prove Proposition \ref{prop:margin-finite-bounds}, we will use the following key key duality result for the worst-case error probability. Note that Lemma \ref{lemma:dual-representation} does not need Assumption \ref{ass:finite-support}.
\begin{lemma}[{\citet[Theorem 1, Eq. 15]{BlanchetMurthy2019}}]\label{lemma:dual-representation}
	For any $w \in \CW$, we have
	\begin{subequations}
		\begin{align}
			\sup_{Q:d_W(P,Q) \leq \epsilon} \bbP_{\xi \sim Q} [z(w,\xi) \leq 0] 
			& = \inf_{t > 0} \left\{ \epsilon t + \bbE_{\xi \sim P}\left[ \max\left\{ 0, 1 - t d(w,\xi)\right\} \right]  \right\}\label{eq:dual-representation}.%
		\end{align}
	\end{subequations}
\end{lemma}
\begin{proof}[Proof of Proposition \ref{prop:margin-finite-bounds}]
	First, by using \eqref{eq:dual-representation} in Lemma~\ref{lemma:dual-representation}, using Assumption~\ref{ass:finite-support} and linear programming duality, we have
	\begin{align*}
		\sup_{Q:d_W(P,Q) \leq \epsilon} \bbP_{\xi \sim Q}[d(w,\xi)=0]  &= \inf_{t > 0} \left\{ \epsilon t + \sum_{i \in [n]} p_i \max\{0,1-t d(w,\xi_i)\} \right\} \\
		& = \max_v \left\{ \sum_{i \in [n]} v_i : \begin{aligned}
			&0 \leq v_i \leq p_i, \ i \in [n]\\
			&\sum_{i \in [n]} d(w,\xi_i) v_i \leq \epsilon
		\end{aligned} \right\}. 
	\end{align*}
	The right-hand side is an instance of a fractional knapsack problem, which is solved by the following greedy algorithm:
	\begin{quote}
		In increasing order of $d(w,\xi_i)$, increase $v_i$ up to $p_i$ or until the budget constraint $\sum_{i \in [n]} d(w,\xi_i) v_i \leq \epsilon$ is tight, whichever occurs first.
	\end{quote}
	Note that when $i \in I(w)$ we have $d(w,\xi_i) = 0$, so we can set $v_i=p_i$ for such values without making a contribution to the knapsack constraint. Hence, the value of the dual program is at least $\sum_{i \in I(w)} p_i$.
	
	When $w \in \CW$ is such that $\epsilon \leq d(w,\xi_i) p_i$ for all $i \in [n] \setminus I(w)$,
	we will not be able to increase any $v_i$ up to $p_i$ for those $i \in [n] \setminus I(w)$ in the dual program. 
	According to the greedy algorithm, we choose the smallest $d(w,\xi_i)$ amongst $i \in [n] \setminus I(w)$ --- whose value corresponds to $\eta(w)$ --- and increase this $v_i$ up to $\epsilon/d(w,\xi_i) = \epsilon/\eta(w) \leq p_i$. 
	Therefore, we have
	\begin{align*}
		\sup_{Q:d_W(P,Q) \leq \epsilon} \bbP_{\xi \sim Q}[d(w,\xi)=0] &= \sum_{i \in I(w)} p_i + \frac{\epsilon}{\eta(w)}.
	\end{align*}
	
	When $w \in \CW$ is such that $\epsilon > d(w,\xi_i) p_i$ for some $i \in [n] \setminus I(w)$, the greedy algorithm for the dual program allows us to increase $v_i$ up to $p_i$ for at least one $i \in [n] \setminus I(w)$. Thus, by similar reasoning to the above, we have that a lower bound on the optimal objective is given by
	\[
	\sum_{i \in I(w)} p_i + \min_{i \in [n] \setminus I(w)} p_i,
	\]
	verifying the second claim.
\end{proof}

The main result in this section, which is a consequence of the first part of this proposition, is that as long as the radius $\epsilon > 0$ is sufficiently small, solving the DRO formulation \eqref{eq:prob-DRO} is equivalent to solving the bilevel optimization problem \eqref{eq:margin-bilevel} for the generalized margin, that is, finding the $w$ that, among those that misclassifies the smallest fraction of points $\bbP_{\xi \sim P}[d(w,\xi)=0] = \rho^*$, achieves the largest margin $\eta(w) = \gamma^*$ on the correctly classified points.
The required threshold for radius $\epsilon$ is  $\epsilon = (\bar{\rho} - \rho^*)\gamma^*$, where $\bar{\rho}$ is the smallest probability that is strictly larger than $\rho^*$, that is,
\begin{equation} \label{eq:def.P}
	\CP := \left\{ \sum_{i \in I} p_i : I \not\in \CI(\rho^*) \right\} = \left\{ \sum_{i \in I} p_i : \sum_{i \in I} p_i > \rho^* \right\}, \quad \bar{\rho} := \min\left\{ \rho : \rho \in \CP  \right\}.
\end{equation}

We show too that classifiers that satisfy $\sum_{i \in I(w)} p_i = \rho^*$ but whose margin may be slightly suboptimal (greater that $\gamma^* - \delta$ but possibly less than $\gamma^*$)
are also nearly optimal for \eqref{eq:prob-DRO}.
\begin{theorem}\label{thm:margin-DRO}
	Let Assumption~\ref{ass:finite-support} be satisfied.
	Suppose that $0 < \epsilon < (\bar{\rho}-\rho^*)\gamma^*$. Then, referring to the DRO problem \eqref{eq:prob-DRO}, we have
	\[ \min_{w \in \CW} \, \sup_{Q:d_W(P,Q) \leq \epsilon} \, \bbP_{\xi \sim Q} \left[ d(w,\xi) = 0 \right] = \rho^* + \frac{\epsilon}{\gamma^*}. \]
	Furthermore, for any $\delta$ with $0 < \delta < \gamma^* - \epsilon/(\bar{\rho} - \rho^*)$, we have
	\begin{align*}
		&\left\{ w \in \CW : I(w) \in \CI(\rho^*), \ \eta(w) \geq \gamma^* - \delta \right\}\\
		&= \left\{ w \in \CW : I(w) \in \CI(\rho^*), \ \sup_{Q:d_W(P,Q) \leq \epsilon} \bbP_{\xi \sim Q} \left[ d(w,\xi) = 0 \right] \leq  \rho^* + \frac{\epsilon}{\gamma^* - \delta} \right\}.
	\end{align*}
	In particular, if there exists some $w \in \CW$ such that $\bbP_{\xi \sim P}[d(w,\xi)=0] = \rho^*$, $\eta(w) = \gamma^*$, then $w$ solves \eqref{eq:prob-DRO}, and vice versa.
\end{theorem}
\begin{proof}
	Since $\sup_{w \in \CW : I(w) \in \CI(\rho^*)} \eta(w) = \gamma(\rho^*) = \gamma^* > \epsilon/(\bar{\rho} - \rho^*)$, there exists some $w \in \CW$ such that $\sum_{i \in I(w)} p_i = \rho^*$ and $\eta(w) > \epsilon/(\bar{\rho} - \rho^*)$, that is, $\epsilon < \eta(w)(\bar{\rho} - \rho^*)$. 
	Now, since $\bar{\rho} - \rho^* \leq p_i$ for all $i \in [n] \setminus I(w)$ (by definition of $\CP$ and $\bar\rho$ in \eqref{eq:def.P}), and since $\eta(w) \leq d(w,\xi_i)$ for all $i \in [n] \setminus I(w)$, we have that $\epsilon < d(w,\xi_i) p_i$ for all $i \in [n] \setminus I(w)$.
	Therefore, by Proposition~\ref{prop:margin-finite-bounds}, we have for this $w$ that
	\[
	\sup_{Q:d_W(P,Q) \leq \epsilon} \bbP_{\xi \sim Q}[d(w,\xi)=0] = \sum_{i \in I(w)} p_i + \frac{\epsilon}{\eta(w)} = \rho^* + \frac{\epsilon}{\eta(w)} < \bar{\rho}.
	\]
	This implies that any $w \in \CW$ such that $\sum_{i \in I(w)} p_i \geq \bar{\rho}$
	is suboptimal for \eqref{eq:prob-DRO}.
	(This is because even when we set $Q=P$ in \eqref{eq:prob-DRO}, such a value of $w$ has a worse objective than the $w$ for which $\sum_{i \in I(w)} p_i=\rho^*$.)
	Furthermore, from Proposition \ref{prop:margin-finite-bounds} and the definition of $\bar{\rho}$, any $w \in \CW$ such that $\sum_{i \in I(w)} p_i = \rho^*$
	and $\epsilon \geq \min_{i \in [n] \setminus I(w)} d(w,\xi_i) p_i$ has
	\[\sup_{Q:d_W(P,Q) \leq \epsilon} \bbP_{\xi \sim Q}[d(w,\xi)=0] \geq \sum_{i \in I(w)} p_i + \min_{i \in [n] \setminus I(w)} p_i = \rho^* + \min_{i \in [n] \setminus I(w)} p_i \geq \bar{\rho},\]
	hence is also suboptimal.
	
	This means that all optimal and near-optimal solutions $w \in \cW$ to \eqref{eq:prob-DRO} with $0 < \epsilon < (\bar{\rho}-\rho^*)\gamma^*$ are in the set
	\[\left\{ w \in \CW : \sum_{i \in I(w)} p_i = \rho^*, \ \epsilon < \min_{i \in [n] \setminus I(w)} d(w,\xi_i) p_i \right\},\]
	and, by Proposition~\ref{prop:margin-finite-bounds}, the objective values corresponding to each such $w$ are
	\[\sup_{Q:d_W(P,Q) \leq \epsilon} \bbP_{\xi \sim Q}[d(w,\xi)=0] = \rho^* + \frac{\epsilon}{\eta(w)}.\]
	By definition of $\gamma(\rho^*) = \gamma^*$, the infimal value for this objective is $\rho^* + \epsilon/\gamma^*$, and it is achieved as $\eta(w) \to \gamma^*$.
	The first claim is proved.
	
	For the second claim, consider any $\delta \in (0, \gamma^* -
	\epsilon/(\bar{\rho} - \rho^*))$. We have for any $w$ with
	$\sum_{i \in I(w)} p_i = \rho^*$ that
	\[
	\eta(w) \geq \gamma^* - \delta \; \iff \; \rho^* + \frac{\epsilon}{\eta(w)} \leq \rho^* + \frac{\epsilon}{\gamma^* - \delta}.
	\]
	Furthermore, for such $w$ and $\delta$, we have $\epsilon < (\gamma^* - \delta)(\bar{\rho} - \rho^*) = (\gamma(\rho^*) - \delta)(\bar{\rho} - \rho^*) \leq \eta(w)(\bar{\rho} - \rho^*)$ so,
	by noting as in the first part of the proof that $\bar{\rho}-\rho^* \le p_i$ and $\eta(w) \le d(w,\xi_i)$ for all $i \in [n] \setminus I(w)$, we have $\epsilon < d(w,\xi_i) p_i$ for all $i \in [n] \setminus I(w)$. By applying Proposition~\ref{prop:margin-finite-bounds} again, we obtain
	\[ 
	\sup_{Q:d_W(P,Q) \leq \epsilon} \bbP_{\xi \sim Q} \left[ z(w,\xi) \leq 0 \right] = \rho^* + \frac{\epsilon}{\eta(w)} \leq \rho^* + \frac{\epsilon}{\gamma^* - \delta}, 
	\]
	as required.
	
	The final claim follows because, using the second claim, we have
	\begin{align*}
		&\left\{ w \in \cW : \bbP_{\xi \sim P}[d(w,\xi)=0] = \rho^*, \eta(w) = \gamma^* \right\}\\
		&= \bigcap_{\delta > 0} \left\{ w \in \CW : I(w) \in \CI(\rho^*), \ \eta(w) \geq \gamma^* - \delta \right\}\\
		& = \bigcap_{\delta > 0} \left\{ w \in \CW : I(w) \in \CI(\rho^*), \ \sup_{Q:d_W(P,Q) \leq \epsilon} \bbP_{\xi \sim Q} \left[ d(w,\xi) = 0 \right] \leq  \rho^* + \frac{\epsilon}{\gamma^* - \delta} \right\} \\
		&= \left\{ w \in \cW : \sup_{Q:d_W(P,Q) \leq \epsilon} \bbP_{\xi \sim Q} \left[ d(w,\xi) = 0 \right] = \rho^* + \frac{\epsilon}{\gamma^*} \right\},
	\end{align*}
	as desired.
	
\end{proof}

Theorem \ref{thm:margin-DRO} shows that, for small Wasserstein ball radius $\epsilon$, the solution of \eqref{eq:prob-DRO} matches the maximum-margin solution of the classification problem, in a well defined sense. 
How does the solution of \eqref{eq:prob-DRO} compare with the minimizer of the more widely used model \eqref{eq:prob-robust}? 
It is not hard to see that when the parameter $\epsilon$ in \eqref{eq:prob-robust} is chosen so that $\epsilon < \gamma^*$, the solution of \eqref{eq:prob-robust} will be a point $w$ with margin $\eta(w) \geq \epsilon$. 
(Such a point will achieve an objective of zero in \eqref{eq:prob-robust}.)
However, in contrast to Theorem~\ref{thm:margin-DRO}, this point may not attain the maximum possible margin $\gamma^*$. 
The margin that we obtain very much depends on the algorithm used to solve \eqref{eq:prob-robust}. 
For fully separable data, for which $\rho^*=0$ and $\gamma^* = \gamma(0) > 0$, %
\citet{CharlesEtAl2019} and \citet{LiEtAl2020Implicit} show that gradient descent applied to a convex approximation of \eqref{eq:prob-robust} with parameter $\eps$ achieves a separation of $\eps$ in an iteration count polynomial in $(\gamma^*-\epsilon)^{-1}$.
Therefore, in order to strengthen the margin guarantee, $\epsilon$ should be taken as close to $\gamma^*$ as possible, but this adversely affects the number of iterations taken to achieve this.
By contrast, Theorem~\ref{thm:margin-DRO} shows that, in the more general setting of non-separable data, the maximum-margin solution is attained from \eqref{eq:prob-DRO} when the parameter $\epsilon$ is taken to be \emph{any} value below the threshold $\gamma^*(\bar{\rho} - \rho^*)$. In particular, this guarantee is algorithm independent.

\subsection{Conditional value-at-risk characterization}\label{sec:margin-cvar}

Section \ref{sec:margin-finite} gives insights into the types of solutions that the distributionally robust model \eqref{eq:prob-DRO} recovers when the Wasserstein radius $\epsilon$ is below a certain threshold. 
When $\epsilon$ is above this threshold however, \eqref{eq:prob-DRO} may no longer yield a maximum-margin solution. 
In this section, we show in Theorem \ref{thm:cvar-minimize-chance-constraint} that, in general, \eqref{eq:prob-DRO} is intimately related to optimizing the conditional value-at-risk of the distance random variable $d(w,\xi)$. 
Thus, when $\epsilon$ is above the threshold of Theorem \ref{thm:margin-DRO}, \eqref{eq:prob-DRO} still has the effect of pushing data points $\xi$ away from the error set $\{\xi \in S : z(w,\xi) \leq 0\}$ as much as possible, thereby encouraging robustness to perturbations. We note that unlike Section \ref{sec:margin-finite}, we make no finite support assumptions on the distribution $P$, that is, Assumption \ref{ass:finite-support} need not hold for our results below.

In stochastic optimization, when outcomes of decisions are random, different risk measures may be used to aggregate these random outcomes into a single measure of desirability (see, for example, \citep{Rockafellar2007risktutorial,BenTalTeboulle2007}).  The most familiar risk measure is expectation.  However, this measure has the drawback of being indifferent between a profit of $1$ and a loss of $-1$ with equal probability, and a profit of $10$ and a loss of $-10$ with equal probability.  In contrast, other risk measures can adjust to different degrees of risk aversion to random outcomes, that is, they can penalize bad outcomes more heavily than good ones.  The conditional value-at-risk (CVaR) is a commonly used measure that captures risk aversion and has several appealing properties.  Roughly speaking, it is the conditional expectation for the $\rho$-quantile of most risky values, for some user-specified $\rho \in (0,1)$ which controls the degree of risk aversion.  Formally, for a non-negative random variable $\nu(\xi)$ where low values are considered risky (that is, ``bad''), CVaR is defined as follows:
\begin{equation} \label{eq:def.cvar}
	\CVaR_\rho(\nu(\xi); P) := \sup_{t > 0} \left\{ t + \frac{1}{\rho} \bbE_{\xi \sim P}\left[ \min\left\{0, \nu(\xi) - t\right\} \right] \right\}. 
\end{equation}

\citet[Corollary~1]{Xie2019} gives a characterization of the chance
constraint
\[
\max_{Q:d_W(P,Q) \leq \epsilon} \bbP_{\xi \sim Q}
[z(w,\xi) \leq 0] \leq \rho
\]
in terms of the CVaR of $d(w,\xi)$ when $P = P_n$, a discrete
distribution.  We provide a slight generalization to arbitrary
$P$.
\begin{lemma}\label{lemma:chance-constraint-cvar}
	Fix $\rho \in (0,1)$ and $\epsilon > 0$. Then, for all $w \in \CW$, we have
	\begin{equation} \label{eq:cc-cvar}
		\sup_{Q:d_W(P,Q) \leq \epsilon} \bbP_{\xi \sim Q} [z(w,\xi) \leq 0] \leq \rho \iff \rho \CVaR_\rho(d(w,\xi); P) \geq \epsilon.
	\end{equation}
\end{lemma}
\begin{proof}
	We prove first the reverse implication in \eqref{eq:cc-cvar}.  Suppose
	that (following \eqref{eq:def.cvar}) we have
	\[
	\rho \CVaR_\rho(d(w,\xi); P) = 
	\sup_{t > 0} \left\{ \rho t + \bbE_{\xi \sim P}\left[ \min\left\{0, d(w,\xi) - t \right\} \right] \right\} \geq \epsilon,
	\]
	then for all $0 < \epsilon' < \epsilon$, there exists some $t > 0$
	such that 
	\[ \rho t + \bbE_{\xi \sim P}\left[ \min\left\{0, d(w,\xi) - t \right\} \right] > \epsilon'.\]
	Dividing by $t$, we obtain $\rho +
	\bbE_{\xi \sim P}\left[ \min\left\{0, d(w,\xi)/t - 1\right\} \right] >
	\epsilon'/t$, so by rearranging and substituting $t'=1/t$, we have
	
	\begin{align*}
		\rho &> \frac{\epsilon'}{t} + \bbE_{\xi \sim P}\left[ \max\left\{0, 1 - \frac{1}{t}d(w,\xi) \right\} \right]\\
		&\ge 
		\inf_{t' > 0} \left\{ \epsilon' t' + \bbE_{\xi \sim P}\left[ \max\left\{ 0, 1 - t' d(w,\xi)\right\} \right] \right\}. 
	\end{align*}
	
	Note that the function
	\[ 
	\epsilon' \mapsto \inf_{t' > 0} \left\{ \epsilon' t' + \bbE_{\xi \sim P}\left[ \max\left\{ 0, 1 - t' d(w,\xi)\right\} \right] \right\} \in [0,\rho]
	\]
	is concave and bounded, hence continuous.  This fact together with the
	previous inequality implies that
	\[ 
	\inf_{t' > 0} \left\{ \epsilon t' + \bbE_{\xi \sim P}\left[
	\max\left\{ 0, 1 - t' d(w,\xi)\right\} \right] \right\} \leq \rho,
	\]
	which, when combined with \eqref{eq:dual-representation}, proves that the reverse implication holds in \eqref{eq:cc-cvar}.
	
	We now prove the forward implication.  Suppose that the left-hand
	condition in \eqref{eq:cc-cvar} is satisfied for some $\rho \in
	(0,1)$, 
	and for contradiction that there exists some $\epsilon' \in (0,\epsilon)$ such that
	\[
	\rho \CVaR_\rho(d(w,\xi); P) = \sup_{t > 0} \left\{ \rho t + \bbE_{\xi \sim P}\left[ \min\left\{0, d(w,\xi) - t \right\} \right] \right\} \le \epsilon' < \epsilon.
	\]
	Then for all $t > 0$, we have
	\begin{align*}
		&\rho t + \bbE_{\xi \sim P}\left[ \min\left\{0, d(w,\xi) - t \right\} \right] \le \epsilon'\\
		&\implies \rho \le \frac{\epsilon'}{t} + \bbE_{\xi \sim P}\left[ \max\left\{0, 1 - \frac{1}{t} d(w,\xi) \right\} \right]\\
		&\implies \rho \leq \inf_{t' > 0} \left\{  \epsilon' t' + \bbE_{\xi \sim P}\left[ \max\left\{0, 1 - t' d(w,\xi) \right\} \right] \right\}.
	\end{align*}
	Since $\epsilon'<\epsilon$, and using the left-hand condition in
	\eqref{eq:cc-cvar} together with \eqref{eq:dual-representation},  we have
	\begin{align}
		\rho &\le \inf_{t > 0} \left\{  \epsilon' t + \bbE_{\xi \sim P}\left[ \max\left\{0, 1 - t d(w,\xi) \right\} \right] \right\}\\
		&\le \inf_{t > 0} \left\{ \epsilon t + \bbE_{\xi \sim P}\left[ \max\left\{0, 1 - t d(w,\xi) \right\} \right] \right\} \le \rho,\\
		\implies \rho &= \inf_{t >
			0} \left\{ \epsilon t + \bbE_{\xi \sim P}\left[ \max\left\{0, 1 - t
		d(w,\xi) \right\} \right] \right\}.\label{eq:rho.expect}
	\end{align}
	Since $\epsilon' < \epsilon$, there cannot exist any $t > 0$ such that
	\[\rho = \epsilon t + \bbE_{\xi \sim P}\left[ \max\left\{0, 1 - t d(w,\xi) \right\} \right].\]
	Let $\rho_k$ and $t_k$ be sequences such that
	$1 > \rho_k > \rho$, $\rho_k \to \rho$, $t_k>0$, and
	\[
	\rho_k \geq \epsilon t_k + \bbE_{\xi \sim P}\left[ \max\left\{0, 1 - t_k d(w,\xi) \right\} \right] > \rho.
	\]
	Since $\epsilon > 0$, there cannot be any subsequence of $t_k$ that
	diverges to $\infty$, since in that case $\epsilon t_k + \bbE_{\xi
		\sim P}\left[ \max\left\{0, 1 - t_k d(w,\xi) \right\} \right] \ge
	\epsilon t_k$ could not be bounded by $\rho_k < 1$. Thus $\{ t_k\}$ is
	bounded, and there exists a convergent subsequence, so we assume
	without loss of generality that $t_k \to \tau$. By the dominated
	convergence theorem, $\bbE_{\xi \sim P}\left[ \max\left\{0, 1 - t_k
	d(w,\xi) \right\} \right] \to \bbE_{\xi \sim P}\left[ \max\left\{0,
	1 - \tau d(w,\xi) \right\} \right]$, and $\epsilon t_k \to \epsilon
	\tau$. But then since
	\[
	\rho < \epsilon t_k + \bbE_{\xi \sim P}\left[ \max\left\{0, 1 - t_k
	d(w,\xi) \right\} \right] \leq \rho_k \to \rho,
	\]
	we have by the squeeze theorem that
	\[
	\epsilon \tau + \bbE_{\xi \sim P}\left[ \max\left\{0, 1 - \tau d(w,\xi) \right\} \right] = \rho.
	\]
	But then, by the fact noted after \eqref{eq:rho.expect}, we must have
	$\tau = 0$ so $\rho = 1$ (from \eqref{eq:rho.expect}), which
	contradicts our assumption that $\rho \in (0,1)$.
\end{proof}

In the case of classification, 
the minimizers of \eqref{eq:prob-DRO} correspond exactly to the
maximizers of $\CVaR_\rho(d(w,\xi);P)$, where $\rho$ is the optimal
worst-case error probability, as we show now.
\begin{theorem}\label{thm:cvar-minimize-chance-constraint}
	Fix some $\rho \in [0,1]$ and define $\epsilon$ (using \eqref{eq:def.cvar}) as follows:
	\begin{equation} \label{eq:def.eps}
		\epsilon := \rho \sup_{w \in \CW} \CVaR_{\rho}(d(w,\xi);P) = \sup_{t > 0} \left\{ \rho t + \bbE_{\xi \sim P}\left[ \min\left\{0, d(w,\xi) - t\right\} \right] \right\}.
	\end{equation}
	If $0 < \epsilon < \infty$, then
	\[ \rho = \inf_{w \in \CW} \sup_{Q:d_W(P,Q) \leq \epsilon} \bbP_{\xi \sim Q} [z(w,\xi) \leq 0]. \]
	Furthermore, the optimal values of $w$ coincide, that is, 
	\[ \argmin_{w \in \CW} \sup_{Q:d_W(P,Q) \leq \epsilon} \bbP_{\xi \sim Q}[z(w,\xi) \leq 0] = \argmax_{w \in \CW} \CVaR_\rho(d(w,\xi);P). \]
\end{theorem}
\begin{proof}
	For any $w \in \CW$ and $t > 0$, we have from \eqref{eq:def.eps} and \eqref{eq:def.cvar} that
	
	\[
	\epsilon \geq \ \sup_{t' > 0} \left\{ \rho t' + \bbE_{\xi \sim P} \left[\min\left\{0,d(w,\xi)-t'\right\}\right] \right\} 
	\geq \rho t + \bbE_{\xi \sim P} \left[\min\left\{0,d(w,\xi)-t\right\}\right]. 
	\]
	
	Dividing by $t$ and rearranging, we obtain
	\[ \frac{\epsilon}{t} + \bbE_{\xi \sim P} \left[ \max\left\{0, 1 - \frac{1}{t} d(w,\xi) \right\} \right] \geq \rho. \]
	Taking the infimum over $t > 0$, using \eqref{eq:dual-representation} (noting that $1/t > 0$),
	then taking the infimum over $w \in \CW$, we obtain
	\begin{equation} \label{eq:ys7}
		\rho \leq \inf_{w \in \CW} \, \sup_{Q:d_W(P,Q) \leq \epsilon} \bbP_{\xi \sim Q} [z(w,\xi) \leq 0].
	\end{equation}
	In the remainder of the proof, we show that equality is obtained in this bound, when $0<\epsilon<\infty$.
	
	Trivially, the inequality in \eqref{eq:ys7} can be replaced by an
	equality when $\rho=1$. We thus consider the case of $\rho<1$, and
	suppose for contradiction that there exists some $\rho' \in (\rho,1]$
	such that for all $w \in \CW$, we have
	\[
	\rho < \rho' < \sup_{Q:d_W(P,Q) \leq \epsilon} \bbP_{\xi \sim Q} [z(w,\xi) \leq 0].
	\]
	It follows from Lemma~\ref{lemma:chance-constraint-cvar} that for all
	$w \in \CW$, we have
	\begin{equation} \label{eq:ys8}
		\sup_{t > 0} \left\{ \rho' t + \bbE_{\xi \sim P} \left[\min\left\{0,d(w,\xi)-t\right\}\right] \right\} < \epsilon.
	\end{equation}
	By taking the supremum over $w \in \CW$ in this bound, and using
	$\rho'>\rho$ and the definition of $\epsilon$ in \eqref{eq:def.eps},
	we have that 
	\begin{align*}
		\epsilon  & \ge \sup_{w \in \CW, t > 0} \left\{ \rho' t + \bbE_{\xi \sim P} \left[\min\left\{0,d(w,\xi)-t\right\}\right] \right\} \\
		& \ge \sup_{w \in \CW, t > 0} \left\{ \rho t + \bbE_{\xi \sim P} \left[\min\left\{0,d(w,\xi)-t\right\}\right] \right\} = \epsilon,
	\end{align*}
	
	so that 
	\begin{align}
		\nonumber
		\epsilon & = \sup_{w \in \CW, t > 0} \left\{ \rho' t + \bbE_{\xi \sim P} \left[\min\left\{0,d(w,\xi)-t\right\}\right] \right\} \\
		\label{eq:ys9}
		& = \sup_{w \in \CW, t > 0} \left\{ \rho t + \bbE_{\xi \sim P} \left[\min\left\{0,d(w,\xi)-t\right\}\right] \right\}.
	\end{align}
	
	From $\rho<\rho'$, \eqref{eq:ys8}, and \eqref{eq:ys9}, we can define
	sequences $\epsilon_k$, $t_k > 0$, and $w_k \in \CW$ such that
	$\epsilon_k \nearrow \epsilon$ and
	\[
	\epsilon_k < \rho t_k + \bbE_{\xi \sim P} \left[\min\left\{0,d(w_k,\xi)-t_k \right\}\right] < \epsilon.
	\]
	By rearranging these inequalities, we obtain
	\begin{align*}
		&\frac{\epsilon_k}{t_k} + \bbE_{\xi \sim P} \left[ \max\left\{0, 1 - \frac{1}{t_k} d(w_k,\xi)\right\} \right]\\
		&\leq \rho < \frac{\epsilon}{t_k} + \bbE_{\xi \sim P} \left[ \max\left\{0, 1 - \frac{1}{t_k} d(w_k,\xi) \right\} \right].
	\end{align*}
	Since $\epsilon_k \to \epsilon$, we have either that $t_k$ is
	bounded away from $0$, in which case
	\[\epsilon/t_k + \bbE_{\xi \sim
		P} [\max\left\{0, 1 - d(w_k,\xi)/t_k\right\}] \to \rho;\]
	or there
	exists a subsequence on which $t_k \to 0$. In the former case,
	we have for $k$ sufficiently large that
	\begin{align*}
		&\frac{\epsilon}{t_k} + \bbE_{\xi \sim P} [\max\left\{0, 1 - d(w_k,\xi)/t_k\right\}] \leq \rho + \frac{\rho'-\rho}{2} < \rho'\\
		&\implies 
		\epsilon < \rho' t_k + \bbE_{\xi \sim P} [\min\left\{0,d(w_k,\xi)-t_k\right\}]\\
		&\implies
		\epsilon < \sup_{w \in \CW} \sup_{t > 0} \left\{ \rho' t + \bbE_{\xi \sim P} [\min\{0,d(w,\xi) - t\}] \right\},
	\end{align*}
	which contradicts \eqref{eq:ys8}.
	We consider now the other case, in which there is a subsequence for
	which $t_k \to 0$, and assume without loss of generality that the full
	sequence has $t_k \to 0$. Since $\bbE_{\xi \sim P} [\min\{0,
	d(w_k,\xi) - t_k\}] \leq 0$ for any $k$, it follows that
	\begin{align*}
		0 &\geq \limsup_{k \to \infty} \left\{ \rho' t_k + \bbE_{\xi \sim P}  [\min\{0, d(w_k,\xi) - t_k\}] \right\}\\
		&\geq \limsup_{k \to \infty} \left\{ \rho t_k + \bbE_{\xi \sim P}  [\min\{0, d(w_k,\xi) - t_k\}] \right\}\\
		&\geq \lim_{k \to \infty} \epsilon_k = \epsilon,
	\end{align*}
	so that $\epsilon \leq 0$.
	This contradicts the assumption that $\epsilon > 0$, so we must have
	\[ \rho = \inf_{w \in \CW} \sup_{Q:d_W(P,Q) \leq \epsilon} \bbP_{\xi \sim Q} [z(w,\xi) \leq 0]. \]
	This completes our proof of the first claim of the theorem.

	Let $w \in \CW$ be a maximizer of the $\CVaR$, so that $\epsilon = \rho \CVaR_\rho(d(w,\xi);P)$.
	Then by Lemma~\ref{lemma:chance-constraint-cvar}, we have
	\[ \sup_{Q:d_W(P,Q) \leq \epsilon} \bbP_{\xi \sim Q} [z(w,\xi) \leq 0] \leq \rho, \]
	so the same value of $w$ is also a minimizer of the worst-case error probability. 
	A similar argument shows that minimizers of the worst-case error probability are also maximizers of the $\CVaR$.
\end{proof}

%% file: reformulation.tex
\section{Reformulation and Algorithms for Linear Classifiers}\label{sec:reformulation}

In this section, we formulate \eqref{eq:prob-DRO} for a common choice
of distance function $c$ and safety function $z$, and discuss
algorithms for solving this formulation.  We make use of the following
assumption.
\begin{assumption}\label{ass:linear-class}
	We have $\CW = \bbR^d \times \bbR$ and $S = \bbR^d \times
	\{\pm 1\}$.
	Write $\bar{w} = (w_0,b_0) \in
	\bbR^d \times \bbR$ and $\xi = (x,y) \in \bbR^d \times \{\pm 1\}$. Define
	$c(\xi,\xi') := \|x - x'\| + \bbI_{y=y'}(y,y')$ for some norm
	$\|\cdot\|$ on $\bbR^d$ and $\bbI_A(\cdot)$ is the convex indicator
	function where $\bbI_A(y,y')=0$ if $(y,y') \in A$ and $\infty$
	otherwise. Furthermore, $z(\bar{w},\xi) := y(w_0^\top \xi + b_0)$.
\end{assumption}
From Lemma \ref{lemma:dual-representation}, the DRO problem \eqref{eq:prob-DRO} is equivalent to
\begin{equation}\label{eq:prob-DRO-transform}
	\inf_{\bar{w} = (w_0,b_0) \in \R^d \times \R, \, t > 0} \left\{ \epsilon t + \bbE_{\xi \sim P} \left[ \max\left\{ 0, 1 - t d(\bar{w},\xi)\right\} \right] \right\}.
\end{equation}
Letting $\|\cdot\|_*$ denote the dual norm of $\|\cdot\|$ from
Assumption \ref{ass:linear-class}, the distance to misclassification
$d(\bar{w},\xi)$ is as follows
\begin{equation} \label{eq:yh9}
	d(\bar{w},\xi) = d((w_0,b_0),(x,y)) =  \begin{cases}
		\frac{\max\left\{0, y(w_0^\top x + b_0) \right\}}{\|w_0\|_*}, &w_0 \neq 0\\
		\infty, &w_0 = 0, \ yb_0 > 0\\
		0, &w_0 = 0, \ yb_0 \leq 0.
	\end{cases} 
\end{equation}

When $w_0 \neq 0$, we can define the following nonlinear transformation:
\begin{equation}\label{eq:nonlinear-transform}
	w \leftarrow \frac{t w_0}{\|w_0\|_*}, \quad b \leftarrow \frac{t b_0}{\|w_0\|_*},
\end{equation}
noting that $t = \|w\|_*$, and substitute \eqref{eq:yh9} into \eqref{eq:prob-DRO-transform} to obtain 
\begin{equation}\label{eq:prob-DRO-linear-transform}
	\inf_{w \in \bbR^d, b \in \bbR} \left\{ \epsilon \|w\|_* + \bbE_{\xi \sim P} \left[ \max\left\{ 0, 1 - \max\left\{0, y(w^\top x + b) \right\} \right\} \right] \right\}.
\end{equation}
In fact, the next result shows that this formulation is equivalent to
\eqref{eq:prob-DRO-transform} even when $w_0=0$. (Here, we use
the term ``$\delta$-optimal solution'' to refer to a point whose
objective value is within $\delta$ of the optimal objective value
for that problem.)
\begin{theorem}
	Under Assumption \ref{ass:linear-class},
	\eqref{eq:prob-DRO-linear-transform} is equivalent to
	\eqref{eq:prob-DRO-transform}.  Moreover, any $\delta$-optimal
	solution $(w,b)$ for \eqref{eq:prob-DRO-linear-transform} can be
	converted into a $\delta$-optimal solution $t$ and $\bar{w}=(w_0,b_0)$
	for \eqref{eq:prob-DRO-transform} as follows:
	\begin{equation} \label{eq:yd5}
		t=\|w\|_*, \quad (w_0,b_0) := \begin{cases}
			\left( \frac{w}{\|w\|_*}, \frac{b}{\|w\|_*} \right) &w \neq 0\\
			(0,b), &w = 0.
		\end{cases} 
	\end{equation}
\end{theorem}
\begin{proof}
	The first part of the proof shows that the optimal value of \eqref{eq:prob-DRO-linear-transform} is less than or equal to that of \eqref{eq:prob-DRO-transform}, while the second part proves the converse.
	
	To prove that the optimal value of \eqref{eq:prob-DRO-linear-transform} is less than or equal to that of \eqref{eq:prob-DRO-transform}, it suffices to show that given any $\bar{w}=(w_0,b_0)$, we can construct a sequence $\{(w^k,b^k)\}_{k \in \bbN}$ such that
	\begin{align}
		\nonumber
		& \epsilon \|w^k\|_* + \bbE_{\xi \sim P} \left[ \max\left\{ 0, 1 - \max\left\{0, y((w^k)^\top x + b^k) \right\} \right\} \right] \\
		\label{eq:yd6}
		& \quad\quad \to \inf_{t \geq 0} \left\{ \epsilon t + \bbE_{\xi \sim P} \left[ \max\left\{ 0, 1 - t d(\bar{w},\xi)\right\} \right] \right\}. 
	\end{align} 
	Consider first the case of $w_0 \neq 0$, and let $t_k>0$ be a sequence
	such that
	\begin{align}
		\nonumber
		& \lim_{k \to \infty} \, \left\{ \epsilon t_k + \bbE_{\xi \sim P} \left[ \max\left\{ 0, 1 - t_k d(\bar{w},\xi)\right\} \right] \right\} \\
		\label{eq:jt2}
		& = \inf_{t > 0} \left\{ \epsilon t + \bbE_{\xi \sim P} \left[ \max\left\{ 0, 1 - t d(\bar{w},\xi)\right\} \right] \right\}.
	\end{align}
	Following \eqref{eq:nonlinear-transform}, we define $w^k := t_k
	w_0/\|w_0\|_*$ and $b^k := t_k b_0 / \|w_0 \|_*$.
	We then have from \eqref{eq:yh9} that
	\begin{align*}
		\max \{ 0, y ( (w^k)^\top x+b^k) \}
		&=
		\max \left\{ 0, t_k \frac{y(w_0^\top x + b_0)}{\| w_0 \|_*} \right\}  \\
		&= t_k \frac{\max \{ 0, y(w_0^\top x + b_0)}{\| w_0 \|_*}
		= t_k d(\bar{w},\xi).
	\end{align*}
	Thus, the left-hand sides of \eqref{eq:jt2} and \eqref{eq:yd6} are
	equivalent, so \eqref{eq:yd6} holds.
	
	Next, we consider the case of $\bar{w}=(w_0,b_0)$ with $w_0 =
	0$.
	Note that $d(\bar{w},\xi) = 0$ when $y b_0 \leq
	0$ and $d(\bar{w},\xi) = \infty$ when $y b_0 > 0$, we have
	$\max\left\{ 0, 1 - t d(\bar{w},\xi) \right\} = \bm{1}(yb \leq 0)$
	for all $t>0$, where $\bm{1}(\cdot)$ has the value $1$ when its
	argument is true and $0$ otherwise. Thus, we have 
	\begin{align}
		\nonumber
		& \inf_{t>0} \left\{ \epsilon t + \bbE_{\xi \sim P} \left[ \max\left\{ 0, 1 - t d(\bar{w},\xi)\right\} \right] \right\} \\
		\label{eq:yh5}
		& = \bbP_{\xi \sim P} [y b_0 \leq 0] = \begin{cases} \bbP_{\xi \sim P} [y \leq 0], &b_0 > 0\\
			1, & b_0=0\\
			\bbP_{\xi \sim P}[y \geq 0], &b_0 < 0.
		\end{cases} 
	\end{align}
	Now choose $w^k=0$ and $b^k = k b_0$ for $k=1,2,\dotsc$. We then have
	\begin{align}
		\nonumber
		& \max\left\{ 0, 1 - \max\left\{0, y((w^k)^\top x + b^k) \right\} \right\} \\
		\nonumber
		&= \max\left\{ 0, 1 - \max\left\{0, k y b_0 \right\} \right\}\\
		\nonumber
		&= \max\left\{ 0, 1 - \max\left\{0, k y b_0 \right\} \right\} \bm{1}(b_0 > 0) + \bm{1}(b_0=0) \\
		\nonumber
		& \quad\quad + \max\left\{ 0, 1 - \max\left\{0, k y b_0 \right\} \right\}\bm{1}(b_0 < 0)\\
		\nonumber
		&= \left( \max\{0,1 - k y b_0 \} \bm{1}(y > 0) + \bm{1}(y \leq 0) \right) \bm{1}(b_0 > 0) + \bm{1}(b_0=0) \\
		\label{eq:yd8}
		& \quad\quad + \left( \bm{1}(y \geq 0) + \max\{0,1 - k y b_0 \} \bm{1}(y < 0) \right) \bm{1}(b_0 < 0).
	\end{align}
	Now notice that for the first and last terms in this last expression,
	we have by taking limits as $k \to \infty$ that
	\begin{align*}
		\left( \max\{0,1 - y b_0 k \} \bm{1}(y > 0) + \bm{1}(y \leq 0) \right) \bm{1}(b_0 > 0) & \to \bm{1}(y \leq 0) \bm{1}(b_0 > 0), \\
		\left( \bm{1}(y \geq 0) + \max\{0,1 - y b_0 k \} \bm{1}(y < 0) \right) \bm{1}(b_0 < 0) & \to \bm{1}(y \geq 0) \bm{1}(b_0 < 0),
	\end{align*}
	both pointwise, and everything is bounded by $1$.  Therefore, by the
	dominated convergence theorem, we have from \eqref{eq:yd8} that
	\begin{equation} \label{eq:yh6}
		\bbE_{\xi \sim P} \left[ \max\left\{ 0, 1 - \max\left\{0, y((w^k)^\top x + b^k) \right\} \right\} \right] \to \begin{cases} \bbP_{\xi \sim P} [y \leq 0], &b_0 > 0\\
			1, & b_0=0\\
			\bbP_{\xi \sim P}[y \geq 0], &b_0 < 0.
		\end{cases}
	\end{equation}
	By comparing \eqref{eq:yh5} with \eqref{eq:yh6}, we see that
	\eqref{eq:yd6} holds for the case of $w_0=0$ too. This completes our
	proof that the optimal value of \eqref{eq:prob-DRO-linear-transform}
	is less than or equal to that of \eqref{eq:prob-DRO-transform}.
	
	We now prove the converse, that the optimal value of
	\eqref{eq:prob-DRO-transform} is less than or equal to that of
	\eqref{eq:prob-DRO-linear-transform}. Given $w$ and $b$, we show
	that there exists $\bar{w} = (w_0,b_0)$ such that
	\begin{align} \nonumber
		& \epsilon \|w\|_* + \bbE_{\xi \sim P} \left[ \max\left\{ 0, 1 - \max\left\{0, y(w^\top x + b) \right\} \right\} \right] \\
		\label{eq:yh1}
		& \geq \inf_{t>0} \left\{ \epsilon t + \bbE_{\xi \sim P} \left[ \max\left\{ 0, 1 - t d(\bar{w},\xi)\right\} \right] \right\}. 
	\end{align}
	When $w \neq 0$, we take $t=\|w\|_*$, $w_0 = w/\|w\|_* = w/t$, and
	$b_0 = b/\|w\|_* = b/t$, and use \eqref{eq:yh9} to obtain \eqref{eq:yh1}.
	
	Specifically, we have
	\begin{align*}
		& \epsilon \|w\|_* - \bbE_{\xi \sim P} \left[ \max\left\{ 0, 1 - \max\left\{0, y(w^\top x + b) \right\} \right\} \right] \\
		&= \epsilon t - \bbE_{\xi \sim P} \left[ \max\left\{ 0, 1 - \max\left\{0,  \frac{ty (w_0^\top x + b_0)}{\|w_0\|_*}  \right\} \right\} \right] \\
		&= \epsilon t - \bbE_{\xi \sim P} \left[ \max\left\{ 0, 1 - t\frac{\max\{0,y (w_0^\top x + b_0)\}}{\|w_0\|_*}  \right\} \right] \\
		&= \epsilon t -  \bbE_{\xi \sim P} \left[ \max\left\{ 0, 1-t d(\bar{w},\xi) \right\} \right] \\
		& \ge \inf_{t>0} \, \left\{  \epsilon t -  \bbE_{\xi \sim P} \left[ \max\left\{ 0, 1-t d(\bar{w},\xi) \right\} \right] \right\},
	\end{align*}
	as claimed.
	
	For the case of $w = 0$, we set $b_0=b$ and obtain
	\begin{align*}
		& \epsilon \|w\|_* + \bbE_{\xi \sim P} \left[ \max\left\{ 0, 1 - \max\left\{ 0, y(w^\top x + b) \right\} \right\} \right] \\
		&= \bbE_{\xi \sim P} \left[ \max\left\{ 0, 1 - \max\left\{ 0, y b \right\} \right\} \right]\\
		&\geq \bbP_{\xi \sim P} \left[ y b \leq 0 \right]
		=\bbP_{\xi \sim P} \left[ y b_0 \leq 0 \right].
	\end{align*}
	By comparing with \eqref{eq:yh5}, we see that \eqref{eq:yh1} holds in this case too. 
	Hence, the objective value of \eqref{eq:prob-DRO-transform} is less than or equal to that of \eqref{eq:prob-DRO-linear-transform}.
	
	For the final claim, we note that the optimal values of the problems
	\eqref{eq:prob-DRO-transform} and \eqref{eq:prob-DRO-linear-transform}
	are equal and, from the second part of the proof above, the transformation \eqref{eq:yd5} gives a solution $t$
	and $\bar{w}=(w_0,b_0)$ whose objective in
	\eqref{eq:prob-DRO-transform} is at most that of $(w,b)$ in
	\eqref{eq:prob-DRO-linear-transform}.  Thus, whenever $(w,b)$ is
	$\delta$-optimal for \eqref{eq:prob-DRO-linear-transform}, then the
	given values of $t$ and $\bar{w}$ are $\delta$-optimal for
	\eqref{eq:prob-DRO-transform}.
\end{proof}

The formulation \eqref{eq:prob-DRO-linear-transform} can be written as the regularized risk minimization problem
\begin{equation}\label{eq:prob-ramploss}
	\inf_{w,b} \, \left\{ \epsilon \|w\|_* + \bbE_{\xi \sim P} \left[ L_R(y(w^\top x + b)) \right] \right\},
\end{equation}
where $L_R$ is the ramp loss function defined by
\begin{equation} \label{eq:ramploss}
	L_R(r) := \max\left\{0,1-r\right\} - \max\{0,-r\} = \begin{cases}
		1, &r \leq 0\\
		1-r, &0 < r < 1\\
		0, &r \geq 1.
	\end{cases}
\end{equation}
Here, the risk of a solution $(w,b)$ is defined to be the expected
ramp loss $\bbE_{\xi \sim P} \left[ L_R(y(w^\top x + b)) \right]$,
and the regularization term $\|w\|_*$ is defined via the norm that is dual to the one introduced in Assumption \ref{ass:linear-class}.

\begin{remark}\label{rem:reformulation-comparison}
	The formulation \eqref{eq:prob-DRO-linear-transform} is
	reminiscent of \citet[Proposition~2]{KuhnEtAl2019tutorial}
	(see also references therein), where other distributionally
	robust risk minimization results were explored, except the
	risk was defined via the expectation of a \emph{continuous and
		convex} loss function, and the reformulation was shown to be
	the regularized risk defined on \emph{the same} loss function.
	In contrast, the risk in \eqref{eq:prob-DRO} is defined as the
	expectation of the \emph{discontinuous and non-convex} $0$-$1$
	loss function $\bm{1}(y (w^\top x + b) \leq 0)$, and the
	resulting reformulation uses the ramp loss $L_R$, a continuous
	but still nonconvex approximation of the $0$-$1$ loss.  \epr
\end{remark}

\begin{remark}\label{rem:ramp-loss}
	The ramp loss $L_R$ has been studied in the context of
	classification by \citet{ShenEtAl2003}, \citet{WuLiu2007}, and
	\citet{CollobertEtAl2006} to find classifiers that are
	robust to outliers.  The reformulation
	\eqref{eq:prob-ramploss} suggests that the ramp loss together
	with a regularization term may have the additional benefit of
	also encouraging robustness to adversarial perturbations in
	the data.  In previous work, there has been several variants
	of ramp loss with different slopes and break points.  The
	formulation \eqref{eq:prob-ramploss-empirical} suggests a
	principled form for ramp loss in classification
	problems.  \epr
\end{remark}

\begin{remark}\label{rem:kernel}
Instead of considering $(w,b) \in \bbR^d \times \bbR$, we may consider non-linear classifiers via kernels (and the associated reproducing kernel Hilbert space-based classifiers). \citet[Section 3.3]{Shafieezadeh-AbadehEtAl2019} examined kernelization of linear classifiers in the context of different Wasserstein DRO-based classification models. They provide approximation results relating the well-known kernel trick to these problems under some assumptions on the kernel $k$. Their results can easily be applied to the ramp loss reformulation \eqref{eq:prob-ramploss} as well.
\epr
\end{remark}

In practice, the distribution $P$ in \eqref{eq:prob-ramploss} is taken to be the empirical distribution $P_n$ on given data points $\{\xi_i\}_{i \in [n]}$, so \eqref{eq:prob-ramploss} becomes
\begin{equation}\label{eq:prob-ramploss-empirical}
	\inf_{w,b} \,  \epsilon \|w\|_* + \frac{1}{n} \sum_{i \in [n]} L_R(y_i(w^\top x_i + b)).
\end{equation}
This problem can be formulated as a mixed-integer program (MIP) and solved to global optimality using off-the-shelf software; see \citep{Brooks2011,BelottiEtAl2016}. 
Despite significant advances in the computational state of the art, the scalability of MIP-based approaches with training set size $m$ remains limited. 
Thus, we consider here an alternative approach based on smooth approximation of $L_R$ and continuous optimization algorithms.

Henceforth, we consider $\|\cdot\| = \|\cdot\|_* = \|\cdot\|_2$ to be the Euclidean norm.
For a given $\epsilon$ in \eqref{eq:prob-ramploss-empirical}, there exists $\beps \ge 0$ such that a strong local minimizer $(w(\epsilon),b(\epsilon))$ of  \eqref{eq:prob-ramploss-empirical} with $w(\epsilon) \neq 0$ is also a strong local minimizer of the following problem:
\begin{equation}\label{eq:prob-ramploss-squarenorm}
	\min_{w,b} \, \frac12 \beps \|w\|^2  + \frac{1}{n} \sum_{i\in [n]} L_R(y_i (w^\top x_i + b) ),
\end{equation}
where we define $\beps = \epsilon/\| w(\epsilon) \|$. In the following
result, we use the notation
\[
g(w,b) := \frac{1}{n} \sum_{i\in [n]} L_R(y_i (w^\top x_i + b) ),
\]
for the summation term in \eqref{eq:prob-ramploss-empirical} and
\eqref{eq:prob-ramploss-squarenorm}.

\begin{theorem} \label{th:w2}
	Suppose that for some $\epsilon>0$, there exists a local minimizer
	$(w(\epsilon),b(\epsilon))$ of \eqref{eq:prob-ramploss-empirical} with
	$w(\epsilon) \neq 0$ and a constant $\tau>0$ such that for all
	$(v,\beta) \in \R^d \times \R$ sufficiently small, we have
	\begin{equation} \label{eq:slm}
		\epsilon \|w(\epsilon)\|+ g(w(\epsilon),b(\epsilon)) + \tau \|v\|^2 
		\le \epsilon \|w(\epsilon)+v\| +  g(w(\epsilon)+v,b(\epsilon)+\beta).
	\end{equation}
	Then for $\beps  = \epsilon/\| w(\epsilon) \|$, $w(\epsilon)$ is also a strong local minimizer of \eqref{eq:prob-ramploss-squarenorm}, in the sense that
	\[
	\frac12 \beps \|w(\epsilon)\|^2 + g(w(\epsilon),b(\epsilon)) + \frac{\tau}{2} \|v\|^2 
	\le \frac12 \beps \| w(\epsilon)+v\|^2 +  g(w(\epsilon)+v,b(\epsilon)+\beta),
	\]
	for all $(v,\beta)$ sufficiently small.
\end{theorem}
\begin{proof}
	For simplicity of notation, we denote $(w,b) = (w(\epsilon),b(\epsilon))$ throughout the proof.
	
	From a Taylor-series approximation of the term $\|w+v \|$, we have
	\begin{align*}
		& \epsilon \|w\| + g(w,b) + \tau \|v\|^2 \\
		& \le \epsilon \|w+v\| + g(w+v,b+\beta) \\
		& =
		\left[ \epsilon \|w\| + \frac{\epsilon}{\|w\|}w^Tv + \frac12 \frac{\epsilon}{\|w\|}v^T \left( I - \frac{ww^T}{w^Tw} \right)v \right]\\
		&\quad + O(\|v\|^3)  + g(w+v,b+\beta) \\
		& \le 
		\left[ \epsilon \|w\| + \frac{\epsilon}{\|w\|}w^Tv + \frac12 \frac{\epsilon}{\|w\|}v^T v \right] + O(\|v\|^3)  + g(w+v,b+\beta) \\
		& = \frac12 \epsilon \|w\| + \frac12 \frac{\epsilon}{\|w\|} (w+v)^T(w+v) + O(\|v\|^3) + g(w+v,b+\beta).
	\end{align*}
	By rearranging this expression, and taking $v$ small enough that the $O(\|v\|^3)$ term is dominated by $(\tau/2) \|v\|^2$, we have
	\[
	\frac12 \epsilon \|w\| + g(w,b) + \frac{\tau}{2} \|v\|^2 \le \frac12 \frac{\epsilon}{\|w\|} (w+v)^T(w+v) + g(w+v,b+\beta).
	\]
	By substituting $\beps = \epsilon/\| w \|$, we obtain the result.
\end{proof}

We note that the condition \eqref{eq:slm} is satisfied when the local
minimizer satisfies a second-order sufficient condition.

To construct a smooth approximation for $L_R(r) = \max\{0,1-r\} -
\max\{0,-r\}$, we follow \citet{BeckTeboulle2012} and approximate the
two max-terms with the softmax operation: For small $\sigma > 0$ and
scalars $\alpha$ and $\beta$,
\[\max\{\alpha,\beta\} \approx \sigma \log\left( \exp\left( \frac{\alpha}{\sigma} \right) + \exp\left( \frac{\beta}{\sigma} \right) \right).\]
Thus, we can approximate $L_R(r)$ by the smooth function $\psisig(r)$,
parametrized by $\sigma>0$ and defined as follows:
\begin{align} 
	\nonumber
	\psisig(r) & := \sigma \log\left( 1 +
	\exp\left( \frac{1-r}{\sigma} \right) \right) - \sigma \log\left(
	1 + \exp\left( -\frac{r}{\sigma} \right) \right) \\
	\label{eq:psisig}
	& = \sigma
	\log\left( \frac{\exp(1/\sigma) + \exp(r/\sigma)}{1 +
		\exp(r/\sigma)} \right).
\end{align}
For any $r \in \R$, we have that
$\lim_{\sigma \downarrow 0} \psisig(r) = L_R(r)$, so the approximation
\eqref{eq:psisig} becomes increasingly accurate as
$\sigma \downarrow 0$. 

By substituting the approximation $\psisig$ in \eqref{eq:psisig} into \eqref{eq:prob-ramploss-squarenorm},
we obtain
\begin{equation}\label{eq:prob-psisig}
	\min_{w,b} \, \left\{ F_{\beps,\sigma}(w) := \frac12 \beps \|w\|^2 + \frac{1}{n} \sum_{i\in [n]} \psisig( y_i (w^\top x_i + b) ) \right\}.  
\end{equation}
This is a smooth nonlinear optimization problem that is nonconvex because
$\psisig''(r) <0$ for $r<1/2$ and $\psisig''(r) >0$ for $r>1/2$.
It can be minimized by any standard method for smooth nonconvex optimization. 
Stochastic gradient approaches with minibatching are best suited to cases in which $n$ is very large. For problems of modest size, methods based on full gradient evaluations are appropriate, such as nonlinear conjugate gradient methods (see \cite[Chapter~5]{NocW06} or L-BFGS \cite{LiuN89}.
Subsampled Newton methods (see for example \cite{bollapragada2019exact,xu2016sub}), in which the gradient is approximated by averaging over a subset of the $n$ terms in the summation in \eqref{eq:prob-psisig} and the Hessian is approximated over a typically smaller subset, may also be appropriate. 
It is well known that these methods are highly unlikely to converge to saddle points,
but they may well converge to local minima of the nonconvex function that are not global minima. 
We show in the next section that, empirically, the global minimum is often found, even for problems involving highly nonseparable data. 
In fact, as  proved in Section~\ref{sec:localmin}, under certain (strong) assumptions on the data, spurious local solutions do not exist.

%% file: numerical.tex
\section{Numerical Experiments}\label{sec:numerical}

We report on computational tests on the linear classification problem
described above, for separable and nonseparable data sets.  We observe
that on separable data, despite the nonconvexity of the problem, the
the smoothed formulation \eqref{eq:prob-psisig} appears to have a
unique local minimizer, found reliably by standard procedures for
smooth nonlinear optimization, for sufficiently large training set
size $n$.  Moreover, the classifier obtained from the ramp loss
formulation is remarkably robust to adversarial perturbations of the
training data: A solution whose classification performance is similar
to the original separating hyperplane is frequently identified even
when a large fraction of the labels from the separable data set are
flipped randomly to incorrect values and when the incorrectly labelled
points are moved further away from the decision boundary.

Our results are intended to be ``proof of concept" in that they both
motivate and support our analysis in Section~\ref{sec:localmin} that
the minimizer of the regularized risk minimization problem
\eqref{eq:prob-ramploss} is the only point satisfying even first-order
conditions, and that the ramp loss can identify classifiers that are robust to perturbations.  Our analysis in Section~\ref{sec:localmin} focuses
  on separable data sets and spherically symmetric distributions, but
  we test here for a non-spherically-symmetric distribution too, and
  also experiment with nonseparable data sets, which are discussed
  only briefly in Section~\ref{sec:localmin}.

\subsection{Test problems, formulation details, and optimization algorithms}\label{sec:numerical-test}

We generate three binary classification problems in which the training
data $(x,y) \sim P$ is such that $x \sim P_x$, where $P_x$ one of
three possible distributions over $\bbR^d$: (1) $N(0,10 I)$; (2) $N(0,
\Sigma)$ where $\Sigma$ is a positive definite matrix with random
orientation whose eigenvalues are log-uniformly distributed in
$[1,10]$; (3) a Laplace distribution with zero mean and covariance
matrix $10 I$. For each $x$, we choose the label $y = \sign\left(
(w^*)^\top x \right)$ where the ``canonical separating hyperplane''
$w^* = (1,0,0,\dotsc,0)$, that is, $y$ is determined by the sign of
the first component of $x$.

We modify this separable data set to obtain nonseparable data sets as
follows. First, we choose a random fraction $\kappa$ of training points
$(x_i,y_i)$ to modify. Within this fraction, we select the points for
which the first component $(x_i)_1$ of $x_i$ is positive, and ``flip''
the label $y_i$ from $+1$ to $-1$. Second, we replace $(x_i)_1$ by
$2(x_i)_1+1$ for these points $i$, moving them further from the
canonical separating hyperplane. In our experiments, we set $\kappa$
to the values $.1$, $.2$ and $.3$.  Since the points $(x_i,y_i)$ for
which $(x_i)_1<0$ are not changed, and $(x_i)_1 < 0$ with probability
$1/2$ for $P_x$ above, the total fractions of of training points that
are altered by this process are (approximately) $.05$, $.1$ and $.15$,
respectively.

We report on computations with the formulation \eqref{eq:prob-psisig}
with $\sigma=.05$ and $\beps=.1$, and various values of $n$.
(The results are not particularly sensitive to the choice of $\sigma$,
except that smaller values yield functions that are less smooth and
thus require more iterations to minimize. The value $\beps=.1$ tends
to yield  solutions $(w,b)$ for which $\|w\| = O(1)$.)

We tried various smooth unconstrained optimization solvers for the
resulting smooth optimization problem --- the PR$+$ version of
nonlinear conjugate gradient~\cite[Chapter~5]{NocW06}, the L-BFGS
method \cite{LiuN89}, and Newton's method with diagonal damping ---
all in conjunction with a line-search procedure that ensures weak
Wolfe conditions.  These methods behaved in a roughly similar manner
and all were effective in finding minimizers.  Our tables report
results obtained only with nonlinear conjugate gradient.

\subsection{Unique local minimizer} \label{sec:numerical-performance}

We performed tests on the separable data sets generated from multiple
instances of the three distributions described above, each from
multiple starting points. Our goal was to determine which instances
appear to have a unique local solution: If the optimization algorithm
converges to the same point from a wide variety of starting points, we
take this observation as empirical evidence that the instance has a
single local minimizer, which is therefore the global minimizer. In
particular, for each distribution and several values of dimension $d$,
we seek the approximate smallest training set size $n$ for which all
instances of that distribution with that dimension appear to have a
single local minimizer. In our experiment, we try $d=5,10,20,40$,
values of $n$ of the form $100 \times 2^i$ for $i=0,1,2,\dotsc$, and
10 different instances generated randomly from each distribution. We
solved \eqref{eq:prob-psisig} for each instance for the hyperplane
$(w,b)$, starting from 20 random points on the unit ball in
$\bbR^{d+1}$. If the same solution is obtained for all 20 starting
points, and this event occurs on all 10 instances, we declare the
corresponding value of $n$ to be the value that yields a unique
minimizer for this distribution and this value of $d$.

\begin{table}[]
	\centering
	\begin{tabular}{l|r|r|r|r}
          \backslashbox{Distribution}{$d$} & 5 & 10 & 20 & 40 \\ \hline
          $N(0,10I)$ & 800 & 1600 & 1600 & 6400 \\
          $N(0,\Sigma)$ & 1600 & 1600 & 3200 & 6400 \\
          $\mbox{\rm Laplace}(0,10I)$ & 1600 & 1600 & 6400 &12800 
	\end{tabular}
	\caption{Approximate training set size $n$ for a problem with
          dimension $d$ to have a single (global) minimizer,
          empirically determined. \label{tab:dvsn}}
\end{table}

The values so obtained are reported in Table~\ref{tab:dvsn}.  We note
that $n$ grows only slowly with $d$, at an approximately linear
rate. These results suggest not only that the underlying loss
\eqref{eq:prob-ramploss-squarenorm} has a single local minimizer,
despite its nonconvexity, but also that this behavior can be observed
for modest training set sizes $n$ in the empirical problem
\eqref{eq:prob-psisig}.

\subsection{Adversarial robustness}\label{sec:numerical-adversarial}

We now explore robustness of classifiers to adversarial perturbations
via the nonseparable data. Note that the ``flipping'' of a point
$(x_i,y_i)$ described in \cref{sec:numerical-test} can be interpreted
as an adversarial perturbation. The motivation behind this is that for
a positively labelled point $(x_i)_1 > 0$, we imagine that the
``correct'' side of the canonical hyperplane is the negative side
$(w^*)^\top x_i \leq 0$ (hence we set $y=-1)$ but we perturb the
$(x_i)_1$ to the ``incorrect'' positive side $(w^*)^\top x_i \geq
0$. For these experiments we fix the dimension to the value $d=10$ and
use $n=10,000$ training points in \eqref{eq:prob-psisig}.  For
comparison, we also solve a model identical to \eqref{eq:prob-psisig}
except that the smoothed ramp-loss function $\psisig$ is replaced by a
smoothed version of the familiar hinge-loss function $L_H(r) =
\max\{0,1-r\}$, which is $\sigma \log\left(1 +
\exp((1-r)/\sigma)\right)$, where again $\sigma=.05$. (Note that the
latter formulation is convex, unlike \eqref{eq:prob-psisig}.)

We measure the performance of the classifier $(w,b)$ obtained from
\eqref{eq:prob-psisig} for various values of the flip fraction
$\kappa$, and the performance of the classifier obtained from the
corresponding empirical hinge-loss objective, in two different ways.
For both methods, we generated 20 random instances of the problem from
each of the three distributions, and measure the outcomes using Monte
Carlo sampling from $n_{\text{test}}=100,000$ test points drawn from
the original separable distribution $P$. In the first method, we
simply calculate the fraction of test points that are misclassified by
$(w,b)$, and calculate the mean and standard deviation of this
quantity over the 20 instances, for each value of $\kappa$.  In the
second method, following \cref{sec:margin-cvar}, we measure the
adversarial robustness of a classifier $(w,b)$ via the conditional
value-at-risk $\CVaR_{\rho}(d((w,b),(x,y));P)$ of the distance
function $d((w,b),(x,y)) = \max\left\{ 0, \frac{y (w^\top x +
  b)}{\|w\|_2} \right\}$ according to \eqref{eq:def.cvar}.  The
empirical value of $\CVaR_{\rho}$ is calculated over the
$n_{\text{test}}=100,000$ test points. A higher value of
$\CVaR_{\rho}$ value means more robustness to perturbations, as the
distances to the classifying hyperplane are larger. Each $\rho \in
[0,1]$ gives a different risk measure, where smaller $\rho$ means that
we focus more on the lower tail of the distribution of $d((w,b),(x,y))$.
We thus compute $\CVaR_{\rho}$ over a range of values of $\rho$ and
compare the \emph{CVaR curves} obtained in this way.

\crefrange{fig:adv-testerror}{fig:adv-cvar} plot our results comparing
ramp and hinge loss for the three distributions. As we increase the
fraction of points flipped, \cref{fig:adv-testerror} shows that the
test error of hinge loss degrades severely, while the test error of
ramp loss is much more stable.  \cref{fig:adv-cvar} shows that the
ramp loss CVaR curve always lies on or above the hinge loss CVaR curve, with
the gap increasing as the fraction of flips increases. These results
provide convincing evidence that the ramp loss leads to more robust
classifiers than the hinge loss.
\input{numerical-plots}

%% file: numerical-plots.tex
\newcommand{\wid}{0.36\textwidth}
\newcommand{\hgt}{0.2}
\newcommand{\scalefactor}{0.85}
\newcommand{\ybound}{0.5}
\pgfplotsset{footnotesize,}
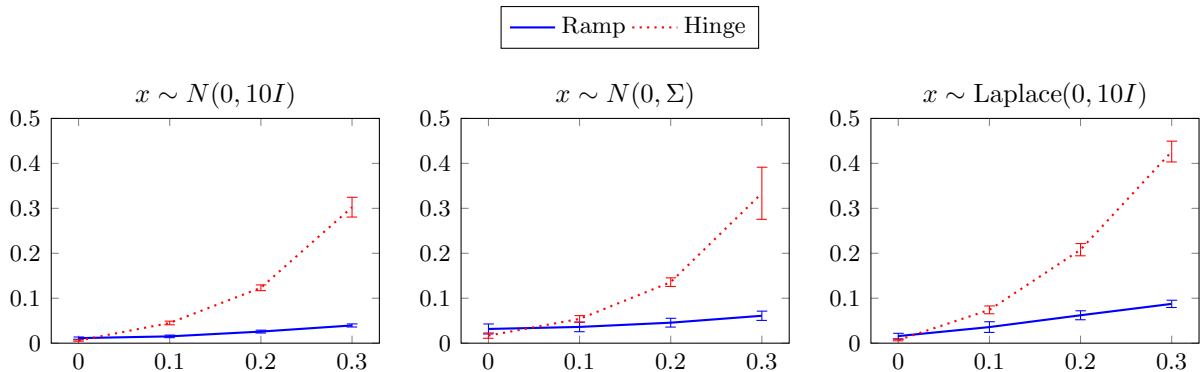
\begin{figure}[htb]
	\centering
	\begin{tikzpicture}
		\begin{axis}[
			at={(0.0\linewidth, 0)},
			width=\wid,height=\hgt\textheight,
			title={$x \sim N(0,10I)$},
			xtick=data,
			yticklabel style={/pgf/number format/fixed},
			yticklabel style={/pgf/number format/fixed},
			ymin=0,ymax=\ybound,
			legend style={at={(0.57\textwidth,1.5)}, legend columns=-1},
			error bars/y dir=both,
			error bars/y fixed,
			error bars/y explicit,
			error bars/error bar style={solid},
			]
			\addplot+[thick, mark=none, color=blue] table[x=frac, y=rampavg_dist1, y error=rampstd_dist1, col sep=comma] {adv_testerror.csv};
			\addplot+[thick, mark=none, color=red, dotted] table[x=frac, y=hingeavg_dist1, y error=hingestd_dist1, col sep=comma] {adv_testerror.csv};
			\legend{Ramp,Hinge}
		\end{axis}
		\begin{axis}[
			at={(0.33\linewidth, 0)},
			width=\wid,height=\hgt\textheight,
			title={$x \sim N(0,\Sigma)$},
			xtick=data,
			yticklabel style={/pgf/number format/fixed},
			yticklabel style={/pgf/number format/fixed},
			ymin=0,ymax=\ybound,
			legend style={legend pos=north west},
			error bars/y dir=both,
			error bars/y fixed,
			error bars/y explicit,
			error bars/error bar style={solid},
			]
			\addplot+[thick, mark=none, color=blue] table[x=frac, y=rampavg_dist2, y error=rampstd_dist2, col sep=comma] {adv_testerror.csv};
			\addplot+[thick, mark=none, color=red, dotted] table[x=frac, y=hingeavg_dist2, y error=hingestd_dist2, col sep=comma] {adv_testerror.csv};
		\end{axis}
		\begin{axis}[
			at={(0.66\linewidth, 0)},
			width=\wid,height=\hgt\textheight,
			title={$x \sim \text{Laplace}(0,10I)$},
			xtick=data,
			yticklabel style={/pgf/number format/fixed},
			yticklabel style={/pgf/number format/fixed},
			ymin=0,ymax=\ybound,
			error bars/y dir=both,
			error bars/y explicit,
			error bars/error bar style={solid},
			]
			\addplot+[thick, mark=none, color=blue] table[x=frac, y=rampavg_dist3, y error=rampstd_dist3, col sep=comma] {adv_testerror.csv};
			\addplot+[thick, mark=none, color=red, dotted] table[x=frac, y=hingeavg_dist3, y error=hingestd_dist3, col sep=comma] {adv_testerror.csv};
		\end{axis}
	\end{tikzpicture}
	\caption{Test error (vertical axis) versus fraction flipped (horizontal axis) for nonseparable data, by distribution type. Note: test error is averaged over 20 trials, with error bars shown for one standard deviation.}
	\label{fig:adv-testerror}
\end{figure}

\renewcommand{\wid}{0.33\textwidth}
\renewcommand{\hgt}{0.15}
\renewcommand{\scalefactor}{0.85}
\renewcommand{\ybound}{2.6}
\pgfplotsset{footnotesize,}

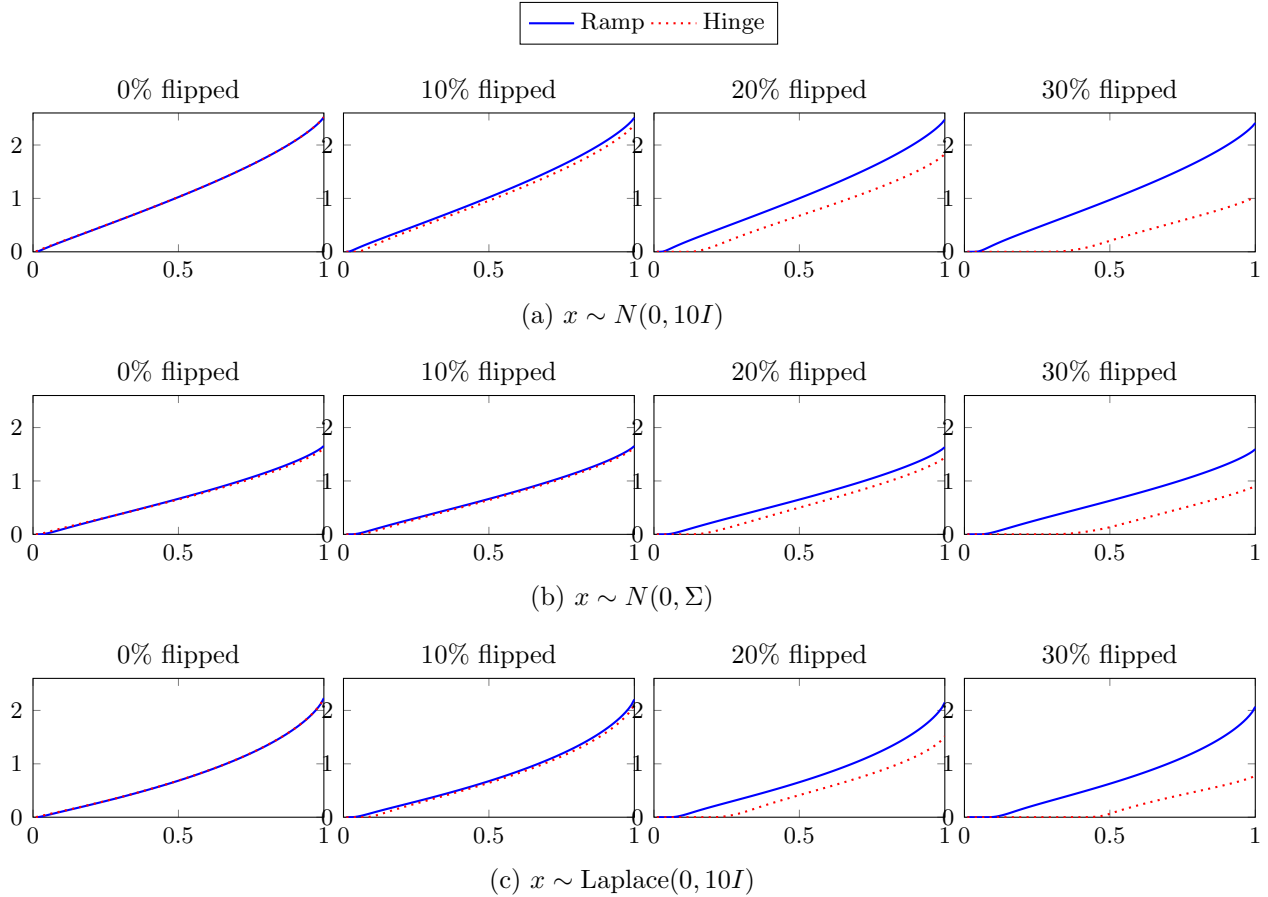
\begin{figure}[htb]
\centering
\captionsetup[subfigure]{position=top}%
\begin{subfigure}{\textwidth}
\centering
\begin{tikzpicture}
	\begin{axis}[
		at={(0, 0)},
		width=\wid,height=\hgt\textheight,
		title={0\% flipped},
		xtick={0,.5,1},
		yticklabel style={/pgf/number format/fixed},
		xmin=0,xmax=1,
		yticklabel style={/pgf/number format/fixed},
		ymin=0,ymax=\ybound,
		legend style={at={(0.6\linewidth,1.8)}, legend columns=-1},
		]
		\addplot [thick, color=blue] table[x=rho, y=rampavg_dist1_0.0, col sep=comma] {adv_cvar.csv};
		\addplot [thick, color=red, dotted] table[x=rho, y=hingeavg_dist1_0.0, col sep=comma] {adv_cvar.csv};
		\legend{Ramp,Hinge}
	\end{axis}
	\begin{axis}[
		at={(0.25\linewidth, 0)},
		width=\wid,height=\hgt\textheight,
		title={10\% flipped},
		xtick={0,.5,1},
		yticklabel style={/pgf/number format/fixed},
		xmin=0,xmax=1,
		yticklabel style={/pgf/number format/fixed},
		ymin=0,ymax=\ybound,
		legend style={at={(0.56\textwidth,1.8)}, legend columns=-1},
		]
		\addplot [thick, color=blue] table[x=rho, y=rampavg_dist1_0.1, col sep=comma] {adv_cvar.csv};
		\addplot [thick, color=red, dotted] table[x=rho, y=hingeavg_dist1_0.1, col sep=comma] {adv_cvar.csv};
	\end{axis}
	\begin{axis}[
		at={(0.5\linewidth, 0)},
		width=\wid,height=\hgt\textheight,
		title={20\% flipped},
		xtick={0,.5,1},
		yticklabel style={/pgf/number format/fixed},
		xmin=0,xmax=1,
		yticklabel style={/pgf/number format/fixed},
		ymin=0,ymax=\ybound,
		legend style={at={(-0.6,-0.2)}, anchor=north,legend columns=-1},
		]
		\addplot [thick, color=blue] table[x=rho, y=rampavg_dist1_0.2, col sep=comma] {adv_cvar.csv};
		\addplot [thick, color=red, dotted] table[x=rho, y=hingeavg_dist1_0.2, col sep=comma] {adv_cvar.csv};
	\end{axis}
	\begin{axis}[
		at={(0.75\linewidth, 0)},
		width=\wid,height=\hgt\textheight,
		title={30\% flipped},
		xtick={0,.5,1},
		yticklabel style={/pgf/number format/fixed},
		xmin=0,xmax=1,
		yticklabel style={/pgf/number format/fixed},
		ymin=0,ymax=\ybound,
		legend style={at={(0.75\textwidth,0.45*\hgt\textheight)}, anchor=north},
		]
		\addplot [thick, color=blue] table[x=rho, y=rampavg_dist1_0.3, col sep=comma] {adv_cvar.csv};
		\addplot [thick, color=red, dotted] table[x=rho, y=hingeavg_dist1_0.3, col sep=comma] {adv_cvar.csv};
	\end{axis}
\end{tikzpicture}
\caption{$x \sim N(0,10I)$}
\label{fig:adv-cvar-dist1}
\end{subfigure}
\begin{subfigure}{\textwidth}
\centering
\begin{tikzpicture}
\begin{axis}[
	at={(0, 0)},
	width=\wid,height=\hgt\textheight,
	title={0\% flipped},
	xtick={0,.5,1},
	yticklabel style={/pgf/number format/fixed},
	xmin=0,xmax=1,
	yticklabel style={/pgf/number format/fixed},
	ymin=0,ymax=\ybound,
	legend style={at={(0.6\linewidth,1.8)}, legend columns=-1},
	]
	\addplot [thick, color=blue] table[x=rho, y=rampavg_dist2_0.0, col sep=comma] {adv_cvar.csv};
	\addplot [thick, color=red, dotted] table[x=rho, y=hingeavg_dist2_0.0, col sep=comma] {adv_cvar.csv};
\end{axis}
\begin{axis}[
	at={(0.25\linewidth, 0)},
	width=\wid,height=\hgt\textheight,
	title={10\% flipped},
	xtick={0,.5,1},
	yticklabel style={/pgf/number format/fixed},
	xmin=0,xmax=1,
	yticklabel style={/pgf/number format/fixed},
	ymin=0,ymax=\ybound,
	legend style={at={(0.56\textwidth,1.8)}, legend columns=-1},
	]
	\addplot [thick, color=blue] table[x=rho, y=rampavg_dist2_0.1, col sep=comma] {adv_cvar.csv};
	\addplot [thick, color=red, dotted] table[x=rho, y=hingeavg_dist2_0.1, col sep=comma] {adv_cvar.csv};
\end{axis}
\begin{axis}[
	at={(0.5\linewidth, 0)},
	width=\wid,height=\hgt\textheight,
	title={20\% flipped},
	xtick={0,.5,1},
	yticklabel style={/pgf/number format/fixed},
	xmin=0,xmax=1,
	yticklabel style={/pgf/number format/fixed},
	ymin=0,ymax=\ybound,
	legend style={at={(-0.6,-0.2)}, anchor=north,legend columns=-1},
	]
	\addplot [thick, color=blue] table[x=rho, y=rampavg_dist2_0.2, col sep=comma] {adv_cvar.csv};
	\addplot [thick, color=red, dotted] table[x=rho, y=hingeavg_dist2_0.2, col sep=comma] {adv_cvar.csv};
\end{axis}
\begin{axis}[
	at={(0.75\linewidth, 0)},
	width=\wid,height=\hgt\textheight,
	title={30\% flipped},
	xtick={0,.5,1},
	yticklabel style={/pgf/number format/fixed},
	xmin=0,xmax=1,
	yticklabel style={/pgf/number format/fixed},
	ymin=0,ymax=\ybound,
	legend style={at={(0.75\textwidth,0.45*\hgt\textheight)}, anchor=north},
	]
	\addplot [thick, color=blue] table[x=rho, y=rampavg_dist2_0.3, col sep=comma] {adv_cvar.csv};
	\addplot [thick, color=red, dotted] table[x=rho, y=hingeavg_dist2_0.3, col sep=comma] {adv_cvar.csv};
\end{axis}
\end{tikzpicture}
\caption{$x \sim N(0,\Sigma)$}
\label{fig:adv-cvar-dist2}
\end{subfigure}
\begin{subfigure}{\textwidth}
\centering
\begin{tikzpicture}
\begin{axis}[
	at={(0, 0)},
	width=\wid,height=\hgt\textheight,
	title={0\% flipped},
	xtick={0,.5,1},
	yticklabel style={/pgf/number format/fixed},
	xmin=0,xmax=1,
	yticklabel style={/pgf/number format/fixed},
	ymin=0,ymax=\ybound,
	legend style={at={(0.6\linewidth,1.8)}, legend columns=-1},
	]
	\addplot [thick, color=blue] table[x=rho, y=rampavg_dist3_0.0, col sep=comma] {adv_cvar.csv};
	\addplot [thick, color=red, dotted] table[x=rho, y=hingeavg_dist3_0.0, col sep=comma] {adv_cvar.csv};
\end{axis}
\begin{axis}[
	at={(0.25\linewidth, 0)},
	width=\wid,height=\hgt\textheight,
	title={10\% flipped},
	xtick={0,.5,1},
	yticklabel style={/pgf/number format/fixed},
	xmin=0,xmax=1,
	yticklabel style={/pgf/number format/fixed},
	ymin=0,ymax=\ybound,
	legend style={at={(0.56\textwidth,1.8)}, legend columns=-1},
	]
	\addplot [thick, color=blue] table[x=rho, y=rampavg_dist3_0.1, col sep=comma] {adv_cvar.csv};
	\addplot [thick, color=red, dotted] table[x=rho, y=hingeavg_dist3_0.1, col sep=comma] {adv_cvar.csv};
\end{axis}
\begin{axis}[
	at={(0.5\linewidth, 0)},
	width=\wid,height=\hgt\textheight,
	title={20\% flipped},
	xtick={0,.5,1},
	yticklabel style={/pgf/number format/fixed},
	xmin=0,xmax=1,
	yticklabel style={/pgf/number format/fixed},
	ymin=0,ymax=\ybound,
	legend style={at={(-0.6,-0.2)}, anchor=north,legend columns=-1},
	]
	\addplot [thick, color=blue] table[x=rho, y=rampavg_dist3_0.2, col sep=comma] {adv_cvar.csv};
	\addplot [thick, color=red, dotted] table[x=rho, y=hingeavg_dist3_0.2, col sep=comma] {adv_cvar.csv};
\end{axis}
\begin{axis}[
	at={(0.75\linewidth, 0)},
	width=\wid,height=\hgt\textheight,
	title={30\% flipped},
	xtick={0,.5,1},
	yticklabel style={/pgf/number format/fixed},
	xmin=0,xmax=1,
	yticklabel style={/pgf/number format/fixed},
	ymin=0,ymax=\ybound,
	legend style={at={(0.75\textwidth,0.45*\hgt\textheight)}, anchor=north},
	]
	\addplot [thick, color=blue] table[x=rho, y=rampavg_dist3_0.3, col sep=comma] {adv_cvar.csv};
	\addplot [thick, color=red, dotted] table[x=rho, y=hingeavg_dist3_0.3, col sep=comma] {adv_cvar.csv};
\end{axis}
\end{tikzpicture}
\caption{$x \sim \text{Laplace}(0,10I)$}
\label{fig:adv-cvar-dist3}
\end{subfigure}
\caption{$\CVaR_\rho(d((w,b),(x,y)); P)$ (vertical axis) versus $\rho$ (horizontal axis) on nonseparable data, by distribution type and fraction flipped. Note: $\CVaR_\rho$ is averaged over 20 trials.}
\label{fig:adv-cvar}
\end{figure}

%% file: localmin.tex
\section{Benign Nonconvexity of Ramp Loss on Linearly Separable Symmetric Data}\label{sec:localmin}

We consider \eqref{eq:prob-ramploss-squarenorm}, setting $b = 0$ for simplicity to obtain
\begin{equation}\label{eq:prob-benign}
	\min_w \left\{ F_\epsilon(w) := \frac12 \epsilon \|w\|_2^2 + \bbE_{(x,y) \sim P} \left[L_R( y w^\top x ) \right] \right\}.
\end{equation}
In this section, we explore the question: is the nonconvex problem
\eqref{eq:prob-benign} \emph{benign}, in the sense that, for
reasonable data sets, descent algorithms for smooth nonlinear
optimization will find the global minimum?  In the formulation
\eqref{eq:prob-benign}, we make use of the true distribution $P$
rather than its empirical approximation $P_n$, because results
obtained for $P$ will carry through to $P_n$ for large $n$, with high
probability.
Exploring this question for general data distributions is difficult,
so we examine \emph{spherically symmetric distributions}.
\begin{definition}\label{def:spherically-symmetric-distribution}
	Let $\Pi$ be a distribution on $\bbR^d$. We say that $\Pi$ is \emph{spherically symmetric about $0$} if, for all measurable sets $A \subset \bbR^d$ and all orthogonal matrices $H \in \bbR^{d \times d}$, we have
	\[ \bbP_{x \sim \Pi} [x \in A] = \bbP_{x \sim \Pi}[Hx \in A]. \]
\end{definition}
Spherically symmetric distributions include normal distributions and Student's $t$-distributions with  covariances $\sigma^2 I$. One useful characterization is that $x = r \cdot s$ where $r$ is a random variable on $\bbR_+$ and $s$ is a uniform random variable on the unit sphere $\{s \in \bbR^d : \|s\|_2 = 1\}$, with $r$ and $s$ independent.

We make the following assumption on the data-generating distribution
\begin{assumption}\label{ass:ramp-loss-symmetric}
The distribution $P$ has the form $y = \sign((w^*)^\top x)$
and $x \sim P_x$ where $P_x$
is some spherically symmetric distribution about $0$ on $\bbR^d$ which is absolutely continuous with respect to Lebesgue measure on $\bbR^d$ (so the probability of lower-dimensional sets is $0$) and $w^*$ is some unit Euclidean norm vector in $\bbR^d$.
\end{assumption}
Under this assumption, we will show that $F_{\epsilon}$ defined in
\eqref{eq:prob-benign} has a single local minimizer $w(\epsilon)$ in
the direction of the canonical hyperplane $w^*$: $w(\epsilon) = \alpha
w^*$ for some $\alpha>0$. Since the function is also bounded below (by
zero) and coercive, this local minimizer is the global minimizer.

We now investigate differentiability properties of the objective $F_\epsilon$.
\begin{lemma} \label{lem:Fe}
	When $w \neq 0$, the function $F_\epsilon(w)$ is differentiable in $w$ with gradient
	\[
	\nabla F_{\epsilon}(w) = \epsilon w - \bbE_{(x,y) \sim P} \left[ \bm{1}\left( 0 \leq y w^\top x \leq 1 \right) y x \right].
	\]
	At $w=0$, the directional derivative of $F_{\epsilon}$ in the direction $w^*$ is $F_\epsilon'(0;w^*) \leq -\bbE_{x \sim P_x} [|(w^*)^\top x|] < 0$.
\end{lemma}
\begin{proof}
We appeal to \citet[Theorem 2.7.2]{Clarke1990nonsmooth} which shows
how to compute the generalized gradient of a function defined via
expectations. We note that for every $(x,y)$, $w \mapsto L_R(y w^\top
x)$ is a regular function since it is a difference of two convex
functions, and is differentiable everywhere except when $y w^\top x
\in \{0,1\}$, with gradient $-\bm{1}(0 < y w^\top x < 1) yx$. When $w
\neq 0$, the set of $(x,y) \sim P$ such that $y w^\top x \in \{0,1\}$
is a measure-zero set under Assumption
  \ref{ass:ramp-loss-symmetric}, so \citet[Theorem
  2.7.2]{Clarke1990nonsmooth} states that the generalized gradient of
$\bbE_{(x,y) \sim P}[L_R(y w^\top x)]$ is the singleton set $\left\{
-\bbE_{(x,y) \sim P} \left[ \bm{1}\left( 0 < y w^\top x < 1 \right) y
  x \right] \right\}$. As it is a singleton, this coincides with the
gradient at $w$. Furthermore, we can write $\bbE_{(x,y) \sim P} \left[ \bm{1}\left( 0 < y w^\top x < 1 \right) y x \right] = \bbE_{(x,y) \sim P} \left[ \bm{1}\left( 0 \leq y w^\top x \leq 1 \right) y x \right]$ since $y w^\top x \in \{0,1\}$ is a set of measure $0$ by \cref{ass:ramp-loss-symmetric}. This proves the first claim.
	
For the final claim, note first that the gradient of the
regularization term $\tfrac12 \epsilon \|w\|_2^2$ is zero at $w=0$.
Thus we need consider only the $L_R$ term in applying the definition
of directional derivative to \eqref{eq:prob-benign}. For the direction
$w^*$, we have
\begin{align*}
		F_{\epsilon}'(0;w^*) &= \lim_{\alpha \downarrow 0} \, \frac{1}{\alpha} \left( \bbE_{(x,y) \sim P} [L_R(y(\alpha (w^*)^\top x))] - \bbE_{(x,y) \sim P} [L_R(0)] \right) \\
		&= \lim_{\alpha \downarrow 0} \, \frac{1}{\alpha} \left( \bbE_{x \sim P_x} [L_R(\alpha |(w^*)^\top x|)] - 1 \right) \\
		&= \lim_{\alpha \downarrow 0} \, \frac{1}{\alpha} \left( \bbE_{x \sim P_x} \left[(1-\alpha |(w^*)^\top x|) \cdot \bm{1}(0 \leq \alpha |(w^*)^\top x| \leq 1) \right] - 1 \right)\\
		&= \lim_{\alpha \downarrow 0} \, \frac{1}{\alpha} \left( \bbP_{x \sim P_x}[0 \leq \alpha |(w^*)^\top x| \leq 1] - 1 \right)\\
		&\quad - \lim_{\alpha \downarrow 0}\, \bbE_{x \sim P_x} \left[|(w^*)^\top x| \cdot \bm{1}(0 \leq \alpha |(w^*)^\top x| \leq 1) \right].
\end{align*}
Now observe that $g_\alpha(x) := |(w^*)^\top x| \bm{1}(0 \leq \alpha |(w^*)^\top x| \leq 1)$ monotonically increases pointwise to $g(x) = |(w^*)^\top x|$  as $\alpha \downarrow 0$, therefore by the monotone convergence theorem $\lim_{\alpha \downarrow 0} \bbE_{x \sim P_x} [|(w^*)^\top x| \bm{1}(0 \leq \alpha |(w^*)^\top x| \leq 1)] = \bbE_{x \sim P_x} [|(w^*)^\top x|]$. Furthermore, $\lim_{\alpha \downarrow 0} \frac{1}{\alpha} \left(\bbP_{x \sim P_x}[0 \leq \alpha |(w^*)^\top x| \leq 1] - 1\right) \leq 0$. Therefore
\[ F_\epsilon'(0;w^*) \leq - \bbE_{x \sim P_x} [|(w^*)^\top x|] < 0. \]
\end{proof}

Lemma~\ref{lem:Fe} shows that $w=0$ is not a local minimum of
$F_\epsilon$, hence any reasonable descent algorithm will not converge
to it.  We now investigate stationary points $\grad F_\epsilon(w) = 0$
for $w \neq 0$ under \cref{ass:ramp-loss-symmetric}.  To this end, we will use the following properties of spherically symmetric distributions.
\begin{lemma}[{\citet[Corollary 4.3]{FourdrinierEtAl2018bookCh4}, \citet[Theorem C.3]{Paolella2018book}}]\label{lemma:spherically-symmetric-conditional-marginal}
	Let $x \sim P_x$ be a spherically symmetric distribution on $\bbR^d$ about $0$. Decompose $x = (x^1,x^2)$ where $x^1 \in \bbR^{p}$ and $x^2 \in \bbR^{d-p}$, with $1 \leq p \leq d-1$. The marginal distribution of $x^1$ and the conditional distribution $x^1 \mid x^2$ are spherically symmetric on $\bbR^p$ about $0$.
\end{lemma}
\begin{lemma}\label{lemma:spherically-symmetric-full-dim-unbounded}
Let $P$ be a spherically symmetric distribution on $\bbR^d$ absolutely
continuous with respect to Lebesgue measure on $\bbR^d$ (i.e., it is a
nondegenerate distribution which has zero measure on any
lower-dimensional set). Consider a closed full-dimensional unbounded
polyhedron $A$ that contains the origin. Then $\bbP_{x \sim P_x}[x \in
  A] > 0$.
\end{lemma}
\begin{proof}
Consider the disjoint union
\[
A = \bigcup_{k \in \bbN} A_k, \quad \text{where} \ A_k = \left\{ x \in
A : k-1 \leq \|x\|_2 < k \right\}.
\]
By our assumptions on $A$, each $A_k$ is non-empty and
full-dimensional. Note that $\bbP_{x \sim P_x}[x \in A_k] \leq \bbP_{x
  \sim P_x}[x \in A] \leq \sum_{k' \in \bbN} \bbP_{x \sim P_x}[x \in
  A_{k'}]$ for every $k \in \bbN$. Note that whenever $\bbP_{x \sim
  P_x}[x \in A_k] = 0$, we must also have $\bbP_{x \sim P_x}[k-1 \leq
  \|x\|_2 < k] = 0$ also, since we can cover $\{x : k-1 \leq \|x\|_2 <
k\}$ with finitely many rotated copies of $A_k$, since it is
full-dimensional, and each of these has identical measure by spherical
symmetry of $P$. Now, if all $\bbP_{x \sim P_x}[x \in A_k] = 0$, then
$\bbP_{x \sim P_x}[x \in \bbR^d] = \sum_{k \in \bbN} \bbP_{x \sim P_x}[k-1 \leq \|x\|_2 < k] = 0$ which is a contradiction. This implies
that there is at least one $\bbP_{x \sim P_x}[x \in A_k] > 0$, hence
$\bbP_{x \sim P_x}[x \in A] > 0$.
\end{proof}
We also use this general property of distributions, which we present without proof. Given a set $A \subseteq \bbR^d$, let $\Conv(A)$ and $\Cone(A)$ be the convex and conic hull respectively.
\begin{lemma}\label{lemma:centroids}
  Let $P_x$ be a distribution over $\bbR^d$. For any measurable set $A
  \subseteq \bbR^d$, $\bbE_{x \sim P_x}[\bm{1}(x \in A) x] = \bbP_{x
    \sim P_x}[x \in A] a \in \Cone(A)$ for some  $a \in \Conv(A)$.
\end{lemma}

We first prove that when $d=2$, points which are not positive multiples of $w^*$ cannot be stationary points. We will then show how the proof for general $d$ essentially reduces to this setting. Since $d=2$, we will write $w = (w_1,w_2)$ and $x = (x_1,x_2)$.
\begin{theorem}\label{thm:2d-ramp-loss-localmin}
Consider $d=2$ and suppose \cref{ass:ramp-loss-symmetric} holds. For
each vector $w \neq (0,0)$ that is not a positive multiple of $w^*$, we have $\grad
F_{\epsilon}(w) \neq (0,0)$.
\end{theorem}
\begin{proof}
From the expression for $\nabla F_\epsilon(w)$ in Lemma~\ref{lem:Fe},
the result will be proved if we can show that
  \begin{equation} \label{eq:EE4}
\begin{aligned}
  & \bbE_{(x,y) \sim P} \left[ \bm{1}(0 \leq yw^\top x \leq 1) yx \right] \\
  & = \bbE_{x \sim P_x} \left[ \bm{1}(0 \leq \sign ((w^*)^Tx) w^\top x \leq 1)  \sign ((w^*)^Tx) x \right]
\end{aligned}
\end{equation}
is not a positive
multiple of $w$ whenever $w$ is not a positive
multiple of $w^*$. Observe that the ``good'' region $\{x : 0 \leq
\sign((w^*)^\top x) w^\top x \leq 1\}$ is the union of two almost
disjoint polyhedra $\{x : (w^*)^\top x \geq 0, 0 \leq w^\top x \leq
1\} \cup \{x : (w^*)^\top x \leq 0, -1 \leq w^\top x \leq 0\}$. Define
\[
\cR := \left\{ x : (w^*)^\top x \geq 0, 0 \leq w^\top x \leq 1 \right\}.
\]
Since $\{x : (w^*)^\top x \leq 0, -1 \leq w^\top x \leq 0\} = \{-x : x
\in \cR\}$ can be obtained by an orthogonal transformation of $\cR$,
we have by spherical symmetry of $P_x$ that
\begin{align*}
		&\bbE_{x \sim P_x} \left[ \bm{1}(0 \leq \sign((w^*)^\top x) w^\top x \leq 1) \sign((w^*)^\top x) x \right]\\
  &= \bbE_{x \sim P_x} \left[ \bm{1}(0 \leq w^\top x \leq 1, (w^*)^\top x \geq 0) x \right] \\
  & \quad - \bbE_{x \sim P_x} \left[ \bm{1}(-1 \leq w^\top x \leq 0, (w^*)^\top x \leq 0) x \right]\\
  &= \bbE_{x \sim P_x} \left[ \bm{1}(0 \leq w^\top x \leq 1, (w^*)^\top x \geq 0) x \right] \\
  & \quad + \bbE_{x \sim P_x} \left[ \bm{1}(0 \leq w^\top (-x) \leq 1, (w^*)^\top (-x) \geq 0) (-x) \right]\\
		&= \bbE_{x \sim P_x} \left[ \bm{1}(x \in \cR) x \right] + \bbE_{x \sim P_x} \left[ \bm{1}(-x \in \cR) (-x) \right] = 2 \bbE_{x \sim P_x} \left[ \bm{1}(x \in \cR) x \right].
\end{align*}
Therefore we need to show that $\bbE_{x \sim P_x} \left[ \bm{1}(x \in \cR) x \right]$ is not a positive multiple
of $w$ whenever $w$ is not a positive multiple of
$w^*$. Since $P_x$ is spherically symmetric, we can without loss of
generality change the basis so that $w^* = (1,0)$, so that $y =
\sign(x_1)$ and $\cR = \left\{ x : x_1 \geq 0, 0 \leq w^\top x \leq 1
\right\}$.
	
Notice that $\sign(x_1) x = (|x_1|,\sign(x_1) x_2)$, so that
\[
\bbE_{x
  \sim P_x} \left[ \bm{1}(0 \leq \sign(x_1) w^\top x \leq 1) \sign(x_1)
  x \right]
\]
has a non-negative first component. Therefore, whenever $w_1 < 0$,
\eqref{eq:EE4} cannot be a positive multiple of $w$.

Consider now the case of $w_1 = 0$.  Since we have already dealt with
the case $w= (0,0)$ in Lemma~\ref{lem:Fe}, and are excluding it from
consideration here, we must have $w_2 \neq 0$. Then
\[
\cR = \left\{ x : x_1 \geq 0, 0 \leq w^\top x \leq 1 \right\} =
\begin{cases}
  \left\{ x : x_1 \geq 0, 0 \leq x_2 \leq 1/|w_2| \right\}, & 
  w_2>0, \\
  \left\{ x : x_1 \geq 0, -1/|w_2| \leq x_2 \leq 0 \right\}, & 
  w_2<0,
\end{cases}
\]
For either sign of $w_2$, since $P_x$ is spherically symmetric and
absolutely continous with respect to Lebesgue measure and $\cR$ is
full-dimensional, unbounded, and contains the origin, by
\cref{lemma:spherically-symmetric-full-dim-unbounded}, $\bbP_{x \sim P_x} \left[x \in \cR \right] > 0$. Additionally, by \cref{ass:ramp-loss-symmetric}, $\bbP_{x \sim P_x} \left[x_1 = 0, x \in \cR \right] = 0$, so that
\[\bbP_{x \sim P_x} \left[x_1 > 0, x \in \cR \right] > 0.\]
Therefore $\bbE_{x \sim P_x}[\bm{1}(x \in \cR) x_1] > 0$,
hence $\bbE_{x \sim P_x} \left[
   \bm{1}(x \in \cR) x \right]$ is not a multiple of $w=(0,w_2)$.
	
We now consider $w_1 > 0$ and without loss of generality $w_2 > 0$. (When $w_2 < 0$, an analogous $\cR$ can be obtained via a reflection across the $x_1$-axis; See \cref{fig:localmin} for an illustration.)
	\begin{figure}[htb]
		\centering
		\begin{tikzpicture}	
			\def\picangle{65}
			\def\rad{2}
			
			\coordinate (O) at (0,0);%
			
			\coordinate (R1) at (0,\rad);%
			\coordinate (S) at +(\picangle:\rad);%
			\coordinate (R2) at +(\picangle-90:\rad);%
			
			\coordinate (R31) at ($(S)!0.8*\rad!90:(O)$);%
			\coordinate[label=above:{$(w^*)^\top x=0$}] (R1s) at +(90:1.5*\rad);%
			\coordinate (R1t) at +(-90:1.5*\rad);%
			\coordinate (R21) at +(\picangle-90:2*\rad);%
			\coordinate (S1) at +(\picangle:1.3*\rad);%
			
			\coordinate (R32) at (intersection of O--R1 and S--R31);%
			
			\path[pattern=crosshatch dots,pattern color=gray!60] (O) --  (R32) -- (R31) -- (R21) -- (O);
			\node at (barycentric cs:O=1,R32=1,R31=1,R21=1) {$\cR$};

			\draw (R1s) -- (R1t);
			\path (O) -- (S1) coordinate[pos=0](Ss) coordinate[pos=1,label=above:{$\Span(\{w\})$}](St);
			\draw (Ss) -- (St);
			\path (O) -- (R21) coordinate[pos=0](R2s) coordinate[pos=1.1,label=right:{$w^\top x=0$}](R2t);
			\draw (R2s) -- (R2t);
			\path (R32) -- (R31) coordinate[pos=0](R3s) coordinate[pos=1.1,label=right:{$w^\top x=1$}](R3t);
			\draw (R3s) -- (R3t);

			\def\picangle{-65}
			\def\rad{2}
			\def\xoffset{6}
			
			\coordinate (O) at (\xoffset,0);%
			
			\coordinate (R1) at (\xoffset,-\rad);%
			
			\coordinate (S) at ($(O)+(\picangle:\rad)$);%
			\coordinate (R2) at ($(O)+(\picangle+90:\rad)$);%
			
			\coordinate (R31) at ($(S)!0.8*\rad!-90:(O)$);%
			\coordinate (R1s) at ($(O)+(-90:1.5*\rad)$);%
			\coordinate[label=above:{$(w^*)^\top x=0$}] (R1t) at ($(O)+(90:1.5*\rad)$);%
			\coordinate (R21) at ($(O)+(\picangle+90:2*\rad)$);%
			\coordinate (S1) at ($(O)+(\picangle:1.3*\rad)$);%
			
			\coordinate (R32) at (intersection of O--R1 and S--R31);%
			
			\path[pattern=crosshatch dots,pattern color=gray!60] (O) --  (R32) -- (R31) -- (R21) -- (O);
			\node at (barycentric cs:O=1,R32=1,R31=1,R21=1) {$\cR$};

			\draw (R1s) -- (R1t);
			\path (O) -- (S1) coordinate[pos=0](Ss) coordinate[pos=1,label=below:$\Span(\{w\})$](St);
			\draw (Ss) -- (St);
			\path (O) -- (R21) coordinate[pos=0](R2s) coordinate[pos=1.1,label=right:{$w^\top x = 0$}](R2t);
			\draw (R2s) -- (R2t);
			\path (R32) -- (R31) coordinate[pos=0](R3s) coordinate[pos=1.1,label=right:{$w^\top x=1$}](R3t);
			\draw (R3s) -- (R3t);
			
			\draw[dashed] (-1,0) -- (\xoffset+5,0);
		\end{tikzpicture}
		\caption{Illustration of $\cR$ when $w_1 > 0$ for $w_2 > 0$ (left) and $w_2 < 0$ (right). Note that the two regions are reflections of one another across the $x_1$-axis (dashed line) when the sign of $w_2$ flips.}
		\label{fig:localmin}
	\end{figure}
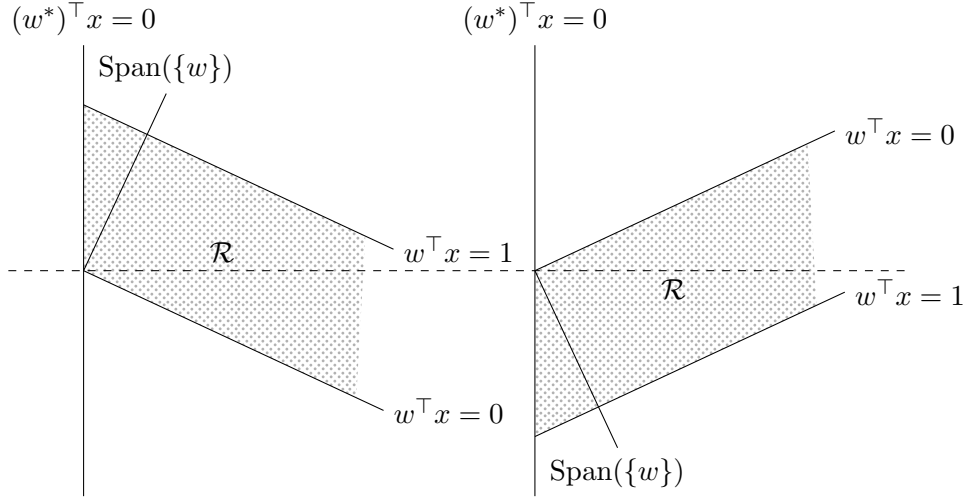	
	We define the lines $R_1,R_2,R_3$ which bound $\cR$, and $S = \Span(\{w\})$ which we will use in our analysis:
	\begin{align*}
		R_1 &= \{ x : (w^*)^\top x = 0\} = \left\{ x : x_1 = 0 \right\}\\
		R_2 &= \{x : w^\top x = 0\}\\
		R_3 &= \{ x : w^\top x = 1\}\\
		S &= \left\{ t w : t \in \bbR \right\}.
	\end{align*}
	Note that $S$ is orthogonal to $R_2$ and $R_3$. We consider the following decomposition of $\cR$:
	\begin{align*}
		\cT &= \text{(closed) triangle bounded by $R_1$, $R_3$ and $S$}\\
		\cT' &= \text{reflection of $\cT$ across $S$}\\
		\cR' &= \cR \setminus (\cT \cup \cT'). 
	\end{align*}
        This decomposition is illustrated in \cref{fig:localmin2}.
	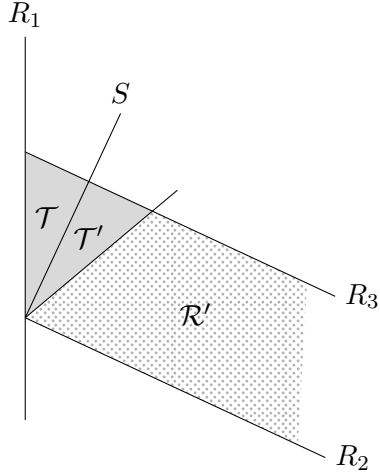
\begin{figure}[tb]
		\centering
		\begin{tikzpicture}	
			\def\picangle{65}
			\def\rad{2}
			
			\coordinate (O) at (0,0) {};%
			
			\coordinate (R1) at (0,\rad) {};%
			\coordinate (S) at +(\picangle:\rad)  {};%
			\coordinate (R2) at +(\picangle-90:\rad)  {};%
			\coordinate (R4) at +(2*\picangle-90:\rad)  {};%
			
			\coordinate (R31) at ($(S)!0.8*\rad!90:(O)$)  {};%
			\coordinate (R11) at +(90:1.7*\rad)  {};%
			\coordinate (R21) at +(\picangle-90:2*\rad)  {};%
			\coordinate (S1) at +(\picangle:1.5*\rad)  {};%
			\coordinate (R41) at +(2*\picangle-90:1.2*\rad)  {};%
			
			\coordinate (R32) at (intersection of O--R1 and S--R31) {};%
			\coordinate (R42) at (intersection of O--R4 and S--R31) {};%
			
			\coordinate (Of) at (0,0);
			\coordinate (Sf) at (S);
			\coordinate (R2f) at (R21);
			\coordinate (R3f1) at (R31);
			\coordinate (R3f2) at (R32);
			\coordinate (R4f) at (R42);
			\path[pattern=crosshatch dots,pattern color=gray!60] (O) --  (R42) -- (R31) -- (R21) -- (O);
			\node at (barycentric cs:O=1,R42=1,R31=1,R21=1) {$\cR'$};
			\path[fill=gray!30] (O) --  (R32) -- (S) -- (O);
			\node at (barycentric cs:O=1,R32=1,S=1) {$\cT$};
			\path[fill=gray!30] (O) --  (S) -- (R42) -- (O);
			\node at (barycentric cs:O=1,S=1,R42=1) {$\cT'$};

			\path (O) -- (R11) coordinate[pos=-0.4](R1s) coordinate[pos=1.1,label=above:$R_1$](R1t);
			\draw (R1s) -- (R1t);
			\path (O) -- (S1) coordinate[pos=0](Ss) coordinate[pos=1,label=above:$S$](St);
			\draw (Ss) -- (St);
			\path (O) -- (R21) coordinate[pos=0](R2s) coordinate[pos=1.1,label=right:$R_2$](R2t);
			\draw (R2s) -- (R2t);
			\path (R32) -- (R31) coordinate[pos=0](R3s) coordinate[pos=1.1,label=right:$R_3$](R3t);
			\draw (R3s) -- (R3t);
			\path (O) -- (R41) coordinate[pos=0](R4s) coordinate[pos=1.1](R4t);
			\draw (R4s) -- (R4t);	
		\end{tikzpicture}
		\caption{Decomposition of $\cR$ into different regions.}
		\label{fig:localmin2}
	\end{figure}
We will now show the following three facts.
\begin{enumerate}
\item We show that $\cT' \subset \cR$, so that in fact $\cR' \cup \cT
  \cup \cT' = \cR$.  To see that $\cT' \subset \cR$, we will show that
  its three extreme points are in $\cR$. These correspond exactly to
  the three extreme points of $\cT$, namely $p_1 = R_1 \cap S =
  (0,0)$, $p_2 = R_1 \cap R_3$ and $p_3 = S \cap R_3$. Clearly the
  reflection of $p_1$ and $p_3$ are themselves since they are already
  on $S$. For $p_2 = R_1 \cap R_3$, we know its reflection $p_2'$ is
  in $R_3$ since $R_3$ is orthogonal to $S$. We check that the first
  coordinate of $p_2'$ is nonnegative in order to deduce that it is in
  $\cR$.
  In fact this claim follows from the fact that $p_3 = t w$ for some $t>0$, from $w_1>0$, and from the explicit formula $p_2' = p_3+ (p_3-p_2)$, which tells that the first component of $p_2'$ is $2tw_1$. Since this value is nonnegative, we are done.

\item $\bbE_{x \sim P_x} \left[ \bm{1}(x \in \cT \cup \cT')x \right] \in
  S$. The distribution $P_x$ is symmetric across $S$ since a reflection
  across a line through the origin is an orthogonal transformation. By
  construction, $\cT \cup \cT'$ is symmetric across $S$, hence
  $\bbE_{x \sim P_x} \left[ \bm{1}(x \in \cT \cup \cT') x \right] \in
  S$.
		
\item We show now that $\bbE_{x \sim P_x} \left[ \bm{1}(x \in \cR') x
  \right] \not\in S$.  Since $P_x$ is spherically symmetric and
  absolutely continous with respect to Lebesgue measure and $\cR'$ is
  full-dimensional, unbounded, and its closure $\cl(\cR')$ contains the origin, by
  \cref{lemma:spherically-symmetric-full-dim-unbounded} we have $0 < \bbP_{x \sim
  	P_x}[x \in \cl(\cR')] = \bbP_{x \sim
    P_x}[x \in \cR']$, where the equality follows by absolute continuity of $P_x$ (\cref{ass:ramp-loss-symmetric}).
  Also, since $(0,0) \in \cT \cup \cT'$, it is
  not in $\cR'$ (but it is an extreme point). Therefore $(0,0) \not\in
  \Conv(\cR')$. By \cref{lemma:centroids}, $\bbE_{x \sim P_x} \left[
    \bm{1}(x \in \cR') x \right] = \bbP_{x \sim P_x}[x \in \cR'] a \neq
  (0,0)$ where $a \in \Conv(\cR')$. Finally, since $\cT$ was defined
  as a triangle with one side on $S$, we clearly have $\cR \cap S
  \subset \cT$. Clearly, $S$ cannot intersect any part of $\cR'$,
  hence $S \cap \Cone(\cR') = \{(0,0)\}$, so that $\bbE_{x \sim P_x} \left[
    \bm{1}(x \in \cR') x \right] \not\in S$.
\end{enumerate}

Since we have
\[
\bbE_{x \sim P_x} \left[ \bm{1}(x \in \cR) x
  \right] = \bbE_{x \sim P_x} \left[ \bm{1}(x \in \cT \cup \cT') x
  \right] + \bbE_{x \sim P_x} \left[ \bm{1}(x \in \cR') x \right],
\]
where the second fact shows that the first vector on the right-hand
side is in $S$ while  the third fact  shows that the second vector on the
right-hand side is not in $S$, we conclude that $\bbE_{x \sim P_x}
\left[ \bm{1}(x \in \cR) x \right] \notin S$, as required.
\end{proof}

We now prove the claim about uniqueness and form of the global minimizer for
the case of general dimension $d$.
\begin{theorem}\label{thm:ramp-loss-localmin}
	For arbitrary dimension $d$, suppose \cref{ass:ramp-loss-symmetric} holds. Then for $w \neq 0$, $\grad F_{\epsilon}(w) \neq 0$ whenever $w$ is not a positive multiple of $w^*$. Furthermore, a unique stationary point $w(\epsilon) = \alpha(\epsilon) w^*$ exists for a unique $\alpha(\epsilon) > 0$.
\end{theorem}
\begin{proof}
	Since $P_x$ is spherically symmetric, without loss of generality, consider $w^* = (1,0,\ldots,0)$. Since $y x_1 = \sign(x_1)x_1 = |x_1|$, we cannot have $w = -\alpha w^*$ for $\alpha > 0$ be a stationary point,
	because
	\begin{align*}
	\grad_{x_1}F_{\epsilon}(-\alpha w^*) &= -\epsilon\alpha - \bbE_{(x,y) \sim P} \left[ \bm{1}(0 \leq y w^\top x \leq 1) y x_1 \right]\\
	&= -\epsilon\alpha - \bbE_{x \sim P_x} \left[ \bm{1}(0 \leq -\alpha |x_1| \leq 1) |x_1| \right] \leq -\eps\alpha < 0.
	\end{align*}
	
	Now consider $w \neq 0$ that is not a multiple of $w^*$. Consider the two-dimensional plane in $\bbR^d$ spanned by $w$ and $w^*$. Change the basis if necessary so that $w = (w_1,w_2,0,\ldots,0)$ (this is without loss of generality as $P_x$ is symmetric hence invariant to orthogonal transformations). With this change of basis, the first two entries of $\bbE_{(x,y) \sim P} \left[ \bm{1}(0 \leq yw^\top x \leq 1) yx \right]$ are determined fully by what happens on the $(x_1,x_2)$ coordinates. More formally, we can without loss of generality consider the marginal distribution $P_x(x_1,x_2)$ on $\bbR^2$ obtained by integrating out $x_3,\ldots,x_d$ (this is spherically symmetric by \cref{lemma:spherically-symmetric-conditional-marginal}). Then \cref{thm:2d-ramp-loss-localmin} applies to prove our first claim.
	
	Finally, consider $w = \alpha w^* = (\alpha,0,\ldots,0)$ for $\alpha > 0$. Then the components of
	\[\bbE_{(x,y) \sim P} \left[ \bm{1}(0 \leq yw^\top x \leq 1) yx \right]\]
	are, for $j \in [d]$,
	\[ \bbE_{(x,y) \sim P} \left[ \bm{1}(0 \leq yw^\top x \leq 1) yx_j \right] = \bbE_{x \sim P_x} \left[ \bm{1}(0 \leq |x_1| \leq 1/\alpha) \sign(x_1) x_j \right]. \]
	By \cref{lemma:spherically-symmetric-conditional-marginal} the conditional distribution $P_x(x_j \mid x_1)$ is still spherically symmetric about $0$ for $j \geq 2$, therefore $\bbE_{x \sim P_x} \left[ \bm{1}(0 \leq |x_1| \leq 1/\alpha) \sign(x_1) x_j \mid x_1 \right] = \bm{1}(0 \leq |x_1| \leq 1/\alpha) \sign(x_1) \bbE_{x \sim P_x} \left[ x_j \mid x_1 \right] = 0$ for any $x_1$. Consequently, for $j \geq 2$, we have
	\[ \bbE_{x \sim P_x} \left[ \bm{1}(0 \leq |x_1| \leq 1/\alpha) \sign(x_1) x_j \right] = 0. \]
	Now consider $j=1$. Then first component of $\bbE_{(x,y) \sim P} \left[ \bm{1}(0 \leq yw^\top x \leq 1) yx \right]$ is 
	\[g(\alpha) := \bbE_{x \sim P_x} \left[ \bm{1}(0 \leq |x_1| \leq 1/\alpha) |x_1| \right].\] Consequently, by \cref{lem:Fe}
	the first component of the gradient is
	\[ \grad_{x_1} F_{\epsilon}(\alpha w^*) = \epsilon \alpha - g(\alpha).\]
	Since $P_x$ is absolutely continuous with respect to Lebesgue measure, $g$ must be continuous. Also, we have $\lim_{\alpha \to \infty} g(\alpha) = 0$ and $\lim_{\alpha \downarrow 0} g(\alpha) = \bbE_{x \sim P_x}[|x_1|] > 0$. Furthermore, $g$ is non-increasing by definition. We conclude that there must exist a unique $\alpha(\epsilon) > 0$ for which $\epsilon \alpha(\epsilon) = g(\alpha(\epsilon))$, and $w(\epsilon) = \alpha(\epsilon) w^*$ is the unique stationary point.
\end{proof}

\paragraph{Label Flipping.}
Our experiments in Section~\ref{sec:numerical} showed that solutions
of the problems analyzed in this section showed remarkable resilience
to ``flipping'' of the labels $y$ on a number of samples.
To give some insight into this phenomenon, suppose that $y = \delta \sign((w^*)^\top x)$ where $\delta \in \{\pm 1\}$ is a random variable independent of $(x,y)$, and $\delta = +1$ (resp. $-1$) with probability $p$ (resp. $1-p$). Then some simple transformations give
\begin{align*}
\bbE_{(x,y) \sim P} \left[ L_R(y w^\top x) \right] &= p \cdot \bbE_{x \sim P_x} \left[ L_R( \sign((w^*)^\top x) w^\top x) \right]\\
&\quad + (1-p) \cdot \bbE_{x \sim P_x} \left[ L_R(- \sign((w^*)^\top x) w^\top x) \right]\\
&= p \cdot \bbE_{x \sim P_x} \left[ L_R( \sign((w^*)^\top x) w^\top x) \right]\\
&\quad + (1-p) \cdot \bbE_{x \sim P_x} \left[ L_R(\sign((-w^*)^\top (-x)) w^\top (-x)) \right]\\
&= p \cdot \bbE_{x \sim P_x} \left[ L_R( \sign((w^*)^\top x) w^\top x) \right]\\
&\quad + (1-p) \cdot \bbE_{x \sim P_x} \left[ L_R(\sign((-w^*)^\top x) w^\top x) \right]
\end{align*}
where the last equality uses the fact that $x$ and $-x$ have the same distribution as $P_x$ is spherically symmetric. We see that in the noisy label setting, the objective $F_{\epsilon}$ is a combination of two noise-free objectives
\begin{align*}
F_{\epsilon}(w) &= p \cdot \left( \epsilon \|w\|_2^2 + \bbE_{x \sim P_x} \left[ L_R( \sign((w^*)^\top x) w^\top x) \right] \right)\\
&\quad + (1-p) \cdot \left( \epsilon \|w\|_2^2 + \bbE_{x \sim P_x} \left[ L_R( \sign((-w^*)^\top x) w^\top x) \right] \right),
\end{align*}
one where labels are $y = \sign((w^*)^\top x)$ with weight $p$, and
the other where labels are $y = \sign((-w^*)^\top x)$ generated by the
\emph{opposite} hyperplane. When $p > 1-p$, more weight is dedicated
to the $w^*$-generated points, hence the solution to
\eqref{eq:prob-benign} is $w^*$, and vice versa. This informal
analysis explains to a large extent the results reported in
Section~\ref{sec:numerical-adversarial}.